\documentclass{article} 
\usepackage{iclr2027_conference,times}

\usepackage{amsmath,amsfonts,bm}

\def\eqref#1{equation~\ref{#1}}

\def\1{\bm{1}}

\DeclareMathAlphabet{\mathsfit}{\encodingdefault}{\sfdefault}{m}{sl}
\SetMathAlphabet{\mathsfit}{bold}{\encodingdefault}{\sfdefault}{bx}{n}

\usepackage{hyperref}
\usepackage{url}
\usepackage{booktabs}       
\usepackage{multirow}
\usepackage{graphicx}
\usepackage{amsmath, amsfonts, amssymb, amsthm}
\usepackage{algorithm}
\usepackage{algorithmic}

\title{Probabilistic Adversarial Training}

\author{%
  Andi Zhang$^{1}$\thanks{Corresponding Author.} \quad Xingyu Zhao$^{1, 2}$ \quad Siddartha Khastgir$^{1}$ \\
  $^1$WMG, University of Warwick \quad
  $^2$Wuhan University\\
  \texttt{andi.zhang@warwick.ac.uk} \\
}

\theoremstyle{plain}
\newtheorem{theorem}{Theorem}
\newtheorem{proposition}[theorem]{Proposition}
\newtheorem{lemma}[theorem]{Lemma}
\newtheorem{corollary}[theorem]{Corollary}
\theoremstyle{definition}

\theoremstyle{remark}

\newcommand{\pdata}{p_{\mathrm{data}}}

\newcommand{\padv}{p_{\mathrm{adv}}}
\newcommand{\pdis}{p_{\mathrm{dis}}}
\newcommand{\pvic}{p_{\mathrm{vic}}}
\newcommand{\ytar}{y_{\mathrm{tar}}}
\newcommand{\yori}{y_{\mathrm{ori}}}

\newcommand{\xori}{x_{\mathrm{ori}}}

\newcommand{\xadv}{x_{\mathrm{adv}}}

\iclrfinalcopy 
\begin{document}

\maketitle

\begin{abstract}
Building on a probabilistic perspective in which adversarial examples arise from the overlap between a distance-based distribution $p_{\mathrm{dis}}$ and a victim-classifier-induced distribution $p_{\mathrm{vic}}$, we start from a simple intuition: adversarial examples become harder to generate when these two distributions are pushed apart, as their overlap becomes smaller, thereby increasing robustness. This intuition naturally motivates a KL-based robustness objective. We then prove that $\mathrm{KL}(p_{\mathrm{dis}}\|p_{\mathrm{vic}})-\log Z_{\mathrm{vic}}$ is a lower bound on probabilistic robustness (PR), where $Z_{\mathrm{vic}}$ denotes the normalizing constant of $p_{\mathrm{vic}}$. Since PR is generally intractable to compute directly, maximizing this KL-based lower bound provides a tractable surrogate objective for improving PR. We further show that this objective recovers a scaled form of adversarial training, offering a probabilistic interpretation of adversarial training and a principled route to robustness improvement. We call the resulting method \textbf{probabilistic adversarial training}. Experiments show that it consistently improves PR, and ablation studies demonstrate that the induced scaling factor can even enhance the PR of non-probabilistic adversarial training methods.
\end{abstract}

\section{Introduction}
\label{sec:introduction}
Adversarial robustness refers to a model’s ability to maintain correct predictions when its inputs are subjected to small perturbations. This property is important in the era of deep learning, as a substantial body of work \citep{goodfellow2014explaining, kurakin2016adversarial, madry2017towards, moosavi2016deepfool, moosavi2017universal, papernot2016limitations, szegedy2013intriguing} has shown that deep neural networks can be highly vulnerable to adversarial perturbations. 

To better understand and evaluate this vulnerability, \citet{zhang2024constructing} proposed probabilistic adversarial attack. By applying Langevin Dynamics, they demonstrated that adversarial examples can be viewed as samples drawn from an adversarial distribution. This distribution is essentially the product of a ``victim'' distribution $p_\mathrm{vic}$ that encourages misclassification, and a ``distance'' distribution $p_\mathrm{dis}$ that constrains the perturbation magnitude. 

However, their framework is strictly confined to targeted attacks and focuses exclusively on the generation of adversarial examples, leaving a systematic approach to model defense unexplored. On the other side, to formally quantify model security, \citet{webb2018statistical} introduced the concept of probabilistic robustness (PR), defined as the probability that a classifier successfully resists perturbations sampled from a specific distribution. 

In this work, we unify the probabilistic attack perspective \citep{zhang2024constructing} with the probabilistic robustness framework \citep{webb2018statistical}. We theoretically prove that $\mathrm{KL}(p_{\mathrm{dis}}\|p_{\mathrm{vic}})-\log Z_{\mathrm{vic}}$ is a lower bound on probabilistic robustness, where $Z_{\mathrm{vic}}$ denotes the normalizing constant of $p_{\mathrm{vic}}$. Since computing and optimizing PR directly is generally intractable, maximizing this KL-based lower bound provides a tractable surrogate objective to effectively enhance PR. 

This theoretical result aligns with a simple yet profound intuition. In essence, adversarial examples naturally arise from the overlap between the distance distribution $p_{\mathrm{dis}}$ (which confines the perturbation to be close to the original image) and the victim distribution $p_{\mathrm{vic}}$ (which actively drives the model toward misclassification). By maximizing the KL divergence $\mathrm{KL}(p_{\mathrm{dis}}\|p_{\mathrm{vic}})$, we are pushing these two distributions apart. As their overlap shrinks, it becomes inherently more difficult for an attacker to sample an adversarial example that simultaneously satisfies the distance constraint and triggers a misclassification. 

We further show that maximizing this lower bound essentially recovers a scaled form of adversarial training. This provides a probabilistic interpretation of adversarial training itself, and provides a principled route to robustness improvement, which we term Probabilistic Adversarial Training (PAT). Empirically, PAT consistently improves probabilistic robustness. Moreover, our ablation studies reveal an intriguing byproduct: the theoretical scaling factor derived from our framework can plug into and enhance traditional, non-probabilistic defense methods, such as PGD-based adversarial training.

To summarize, our main contributions are as follows:
\begin{itemize}
    \item We formalize the untargeted version of probabilistic adversarial attacks. (Section~\ref{sec:untargeted})
    \item We prove that $\mathrm{KL}(p_{\mathrm{dis}}\|p_{\mathrm{vic}})-\log Z_{\mathrm{vic}}$ is a lower bound on probabilistic robustness (PR). (Section~\ref{sec:lowerbound})
    \item By maximizing this KL-based lower bound, we propose Probabilistic Adversarial Training (PAT), which is theoretically guaranteed to improve PR. Moreover, its importance sampling formulation makes the estimation of the bound tractable (Section~\ref{sec:pat}).
    \item We empirically demonstrate that PAT consistently improves PR, and its derived scaling factor can even enhance non-probabilistic methods like PGD-based adversarial training. (Section~\ref{sec:exps})
\end{itemize}

\section{Preliminaries}

\subsection{Adversarial Attack}
Let \(C: [0,1]^d \rightarrow \mathcal{Y}\) be a classifier with input dimension \(d\) and label space \(\mathcal{Y}\). Let \(\xori \in [0,1]^d\) be a clean (original) image with its corresponding true label \(\yori \in \mathcal{Y}\). Adversarial attacks generally fall into two categories. The goal of an \textbf{untargeted adversarial attack} is to construct an adversarial example \(\xadv\) that simply causes the classifier to misclassify the image, meaning \(C(\xadv) \neq \yori\). Alternatively, given a specific target label \(\ytar \in \mathcal{Y}\) (where \(\ytar \neq \yori\)), a \textbf{targeted adversarial attack} aims to force the classifier to predict that exact label, such that \(C(\xadv) = \ytar\). In both scenarios, the distance between \(\xadv\) and \(\xori\) must remain small. The corresponding optimization problem can be formulated as:
\[
    \min \mathcal{D}(\xori, \xadv)
    \quad \text{subject to} \quad
    C(\xadv) \in \mathcal{T} \quad \text{and} \quad
    \xadv \in [0, 1]^d,
\]
where $\mathcal{T} = \{\ytar\}$ for a targeted attack and $\mathcal{T} = \mathcal{Y} \setminus \{\yori\}$ for an untargeted attack, and $\mathcal{D}$ measures the distance (similarity) between $\xori$ and $\xadv$, typically via an $L_1$, $L_2$, or $L_\infty$ norm. Directly solving this constrained optimization can be challenging. To address this, \citet{szegedy2013intriguing} propose relaxing it into the following optimization problem:
\begin{equation}
\label{eq:finalopti}
    \min_{\xadv \in [0,1]^d} \;c_1\,\mathcal{D}(\xori, \xadv) \;+\; c_2\,L(\xadv)
\end{equation}
where $c_1>0$ and $c_2>0$ are trade-off constants. For a targeted attack, we set $L(\xadv) = f(\xadv, \ytar)$ to minimize the loss with respect to the target label; for an untargeted attack, we set $L(\xadv) = -f(\xadv, \yori)$ to maximize the loss with respect to the true label, thereby pushing the prediction away from the original class. In \cite{szegedy2013intriguing}'s work, $f$ is taken to be the cross-entropy loss\footnote{When we say that \(f(x,y)\) is the cross-entropy loss, we mean the cross-entropy between the Dirac categorical distribution concentrated at label \(y\) and the categorical distribution output by the neural network given input \(x\).}; \citet{carlini2017towards} present additional choices for $f$.

\subsection{Probabilistic Adversarial Attack}
\label{sec:probattack}
\citet{zhang2024constructing} derive a probabilistic perspective on adversarial attacks by applying Langevin Dynamics as an optimizer for \eqref{eq:finalopti}. Notably, their framework focuses exclusively on targeted attack. Under this specific setting (where \(L(\xadv) = f(\xadv, \ytar)\)), they introduce the adversarial distribution:
\[
 \padv(\xadv ; \xori, \ytar) \propto \pvic(\xadv ; \ytar) \, \pdis(\xadv ; \xori),
\]
where \(\pvic(\xadv ; \ytar) \propto \exp\bigl(-c_2 \, f(\xadv, \ytar)\bigr)\) is the ``victim'' distribution emphasizing misclassification toward the target label \(\ytar\), and \(\pdis(\xadv ; \xori) \propto \exp\bigl(-c_1\,\mathcal{D}(\xori, \xadv)\bigr)\) is the ``distance'' distribution around \(\xori\). This formulation leverages the fact that Langevin Dynamics converges to the corresponding Gibbs distribution \citep{lamperski2021projected}, thereby providing a probabilistic interpretation of adversarial examples.

This probabilistic perspective aligns with traditional geometry-based adversarial attacks. For example, if \(\mathcal{D}\) is the \(L_1\) norm, then \(\pdis(\xadv ; \xori) \propto \exp(-\lVert \xadv - \xori\rVert_{1})\) takes the form of a Laplace distribution; if \(\mathcal{D}\) is the squared \(L_2\) norm, then \(\pdis(\xadv ; \xori) \propto \exp(-\lVert \xadv - \xori\rVert_{2}^2)\) is a Gaussian distribution; if \(\mathcal{D}\) is \(L_\infty\) norm, then $\pdis$ is an \(L_\infty\)-spherical distribution \citep{iglesias1998characterizations}.

\subsection{Probabilistic Robustness}
\label{sec:probrobust}
\citet{webb2018statistical} introduced a statistical framework to assess neural network robustness by evaluating the probability of an attack's success. Let $p$ denote the probability density function of perturbations around a clean input $x_{\mathrm{ori}}$. To quantify misclassification, let $z(x) \in \mathbb{R}^{|\mathcal{Y}|}$ denote the pre-softmax logits output by the classifier. For an untargeted attack, the margin function $s(\cdot)$ is formally defined as the logit margin violation:
\begin{equation*}
    s(x) = \max_{y \neq y_{\mathrm{ori}}} \left[ z_y(x) - z_{y_{\mathrm{ori}}}(x) \right].
\end{equation*}
The event $s(X) \ge 0$ strictly indicates that at least one incorrect class logit matches or exceeds the true class logit, signifying a successful adversarial attack (i.e., the classifier is fooled). 

Consequently, the attack success probability, or \textbf{probabilistic vulnerability}, is defined as:
\begin{equation*}
    \mathcal{I}[p, s] = \mathbb{P}_{X \sim p}(s(X) \ge 0) = \int \mathbb{I}_{\{s(x) \ge 0\}} p(x) dx,
\end{equation*}
where $\mathbb{I}$ is the indicator function. 

While directly minimizing $\mathcal{I}[p,s]$ reduces vulnerability, framing robustness as an objective to be maximized provides a more consistent narrative for our subsequent theoretical derivations. Therefore, throughout this paper, we formally define \textbf{Probabilistic Robustness} (PR) as the complementary probability of an attack failing:
\begin{equation*}
    \mathrm{PR}(p, s) := 1 - \mathcal{I}[p, s] = \mathbb{P}_{X \sim p}(s(X) < 0).
\end{equation*}
Under this definition, enhancing the intrinsic security of a model corresponds directly to maximizing the value of PR.

\subsection{Adversarial Training}

Adversarial training (AT) is a widely adopted empirical method to defend against adversarial attack described above \citep{madry2017towards}. The core idea of AT is to augment the training data with adversarial examples during the learning process. Formally, let $\theta$ denote the parameters of the classifier, and let $\pdata$ be the underlying data distribution. Adversarial training formulates the robustness objective as a min-max optimization problem:
\begin{equation*}
    \min_{\theta} \mathbb{E}_{(\xori, \yori) \sim \pdata} \left[ \max_{x_{\mathrm{adv}} \in \mathcal{S}(\xori)} f(x_{\mathrm{adv}}, \yori; \theta) \right],
\end{equation*}
where $\mathcal{S}(\xori) = \{x_{\mathrm{adv}} \in [0,1]^{d} : \mathcal{D}(\xori, \xadv) \le \epsilon\}$ defines the set of allowed adversarial perturbations bounded by a small budget $\epsilon$ (e.g., an $L_{\infty}$ or $L_{2}$ norm ball). 

\section{An Untargeted Version of Probabilistic Adversarial Attack}
\label{sec:untargeted}
As introduced in Section~\ref{sec:probattack}, an untargeted adversarial attack can be formulated as the following optimization problem:
\begin{equation*}
    \min_{\xadv \in [0,1]^d} \; c_1\,\mathcal{D}(\xori, \xadv) \;-\; c_2 \,f(\xadv, \yori)
\end{equation*}
Similar to the work of \citet{zhang2024constructing}, applying Langevin dynamics to this optimization problem yields convergence to a Gibbs distribution concentrated around the solutions:
\begin{equation*}
 \padv(\xadv ; \xori, \yori) \propto \pvic(\xadv ; \yori) \, \pdis(\xadv ; \xori),
\end{equation*}
where \(\pvic(\xadv ; \yori) \propto \exp\bigl(c_2 \, f(\xadv, \yori)\bigr)\) is the ``victim'' distribution and \(\pdis(\xadv ; \xori) \propto \exp\bigl(-c_1\,\mathcal{D}(\xori, \xadv)\bigr)\) is the ``distance'' distribution around \(\xori\). For typical choices of \(\mathcal{D}\), it is clear that \(\pdis\) is well-defined \citep{zhang2024constructing}. The following proposition shows that \(\pvic\) is also well-defined in the untargeted setting:

\begin{proposition}[Well-definedness of the untargeted victim distribution]
Let \(K=[0,1]^d\) and let \(f(\cdot,y_{\mathrm{ori}}):K\to\mathbb{R}\) be continuous. For any \(c>0\), define
\[
    p_{\mathrm{vic}}(x;y_{\mathrm{ori}})
    =
    \frac{\exp\big(c f(x,y_{\mathrm{ori}})\big)}
    {\int_K \exp\big(c f(u,y_{\mathrm{ori}})\big)\,du},
    \qquad x\in K.
\]
Then \(p_{\mathrm{vic}}(\cdot;y_{\mathrm{ori}})\) is a well-defined probability density on \(K\).
\end{proposition}
The proof is in Appendix~\ref{app:proof}. In this paper, we take \(\mathcal{D}\) to be the squared \(L_2\) norm, so that \(\pdis\) is a Gaussian distribution, and choose \(f\) to be the cross-entropy loss. Under this setting, the Langevin-dynamics-based algorithm for sampling adversarial examples from \(\padv\) is presented in Algorithm~\ref{alg:untargetedatk} (See Appendix~\ref{app:code}).

\section{A KL-Based Lower Bound on Probabilistic Robustness}
\label{sec:lowerbound}
Recall that $\mathrm{PR}(p, s)$ (introduced in Section~\ref{sec:probrobust}) evaluates the model's security against perturbations drawn from a distribution $p$ around each clean (original) input $\xori$. We observe that the distance distribution $p_{\mathrm{dis}}(\cdot; x_{\mathrm{ori}})$ introduced in Section~\ref{sec:probattack} naturally serves as the evaluation distribution $p$. Consequently, our objective formally becomes maximizing $\mathrm{PR}(p_{\mathrm{dis}}, s)$.

To establish a mathematical lower bound on PR, we must connect the attack-success event, characterized by the logit margin violation $s(x) \ge 0$ (Section~\ref{sec:probrobust}), with our optimization objective, which relies on the continuous cross-entropy loss $f(x, y_{\mathrm{ori}})$. However, establishing a direct, tractable algebraic mapping between the margin $s(x)$ and the loss $f$ is highly challenging due to softmax non-linearities and complex multi-class dynamics. 

To circumvent this, we define the attack-success region $E_{\mathrm{adv}} = \{x \in \mathcal{X} : s(x) \ge 0\}$ as the vulnerable space where the classifier is fooled. We then introduce a connecting scalar $\gamma_{*}$ representing the infimum of the scaled cross-entropy loss over this region:
\begin{equation*}
    \gamma_{*} := \inf_{x \in E_{\mathrm{adv}}} c_2 f(x, y_{\mathrm{ori}}; \theta).
\end{equation*}
Since the cross-entropy loss is inherently non-negative and the temperature scaling factor $c_2 > 0$, it strictly holds that $\gamma_{*} \ge 0$. By definition, any successful adversarial example $x \in E_{\mathrm{adv}}$ rigorously guarantees $c_2 f(x, y_{\mathrm{ori}}; \theta) \ge \gamma_{*}$. This property is pivotal, as it allows us to formally link the expected loss under $p_{\mathrm{dis}}$ to the probability mass of the vulnerable region, paving the way for the following theorem.

\begin{theorem}[KL-Based Lower Bound on Probabilistic Robustness] \label{thm:kl_bound}
Let $p_{\mathrm{dis}}(\cdot; x_{\mathrm{ori}})$ be a probability density on $[0,1]^{d}$. Suppose the victim distribution parameterized by $\theta$ is defined as
\begin{equation*}
    p_{\mathrm{vic}}(x; y_{\mathrm{ori}}, \theta) = \frac{\exp(c_2 f(x, y_{\mathrm{ori}}; \theta))}{Z_{\mathrm{vic}}},
\end{equation*}
where $Z_{\mathrm{vic}} = \int_{\mathcal{X}} \exp(c_2 f(u, y_{\mathrm{ori}}; \theta)) du$ is the normalizing constant. Let $\gamma_{*}$ be defined as above and assume $\gamma_{*}>0$. Then, the probabilistic robustness satisfies:
\begin{equation*}
    \mathrm{PR}(p_{\mathrm{dis}}, s) \ge 1 - \frac{\log Z_{\mathrm{vic}} - H(p_{\mathrm{dis}}) - \mathrm{KL}(p_{\mathrm{dis}} \| p_{\mathrm{vic}})}{\gamma_{*}},
\end{equation*}
where $H(p_{\mathrm{dis}}) = -\int_{\mathcal{X}} p_{\mathrm{dis}}(x) \log p_{\mathrm{dis}}(x) dx$ is the differential entropy of the distance distribution.
\end{theorem}


The proof is in Appendix~\ref{app:proof}. Crucially, Theorem \ref{thm:kl_bound} formally unifies the probabilistic attack perspective \citep{zhang2024constructing} with the probabilistic robustness framework \citep{webb2018statistical}. By expressing the lower bound on PR using the statistical divergence between $p_{\mathrm{dis}}$ and $p_{\mathrm{vic}}$, this theorem directly translates our intuition of ``pushing the distributions apart'' into a rigorous mathematical objective. 

Notice that once the clean input $\xori$ and the perturbation budget are fixed, the distance distribution $p_{\mathrm{dis}}$ and its differential entropy $H(p_{\mathrm{dis}})$ are constants. Furthermore, since $\gamma_{*}$ acts as a positive scaling factor, maximizing the lower bound on PR requires maximizing the term $\mathrm{KL}(p_{\mathrm{dis}} \| p_{\mathrm{vic}}) - \log Z_{\mathrm{vic}}$. This fundamental insight serves as the theoretical bedrock for the tractable surrogate objective derived in the next section. 

While intuitively one might assume that maximizing $KL(p_{dis}||p_{vic})$ alone is sufficient to separate the distributions, the $-\log Z_{vic}$ term holds crucial geometric and physical significance. We defer a detailed discussion on this to Appendix~\ref{app:discussion}.

\section{Probabilistic Adversarial Training}
\label{sec:pat}
Based on Section 4, maximizing the lower bound $KL(p_{dis}||p_{vic})-\log Z_{vic}$ acts as a surrogate objective that indirectly improves probabilistic robustness. However, calculating the exact gradients of this theoretical bound is challenging due to intractable normalizing constants. In this section, we derive a computable empirical loss function through a structured, three-step derivation. We then explicitly reveal how this empirical formulation recovers a scaled, probabilistic counterpart to standard adversarial training, before formally proving its convergence. Algorithm~\ref{alg:full_training} in Appendix~\ref{app:code} presents the pseudocode for the final practical implementation of probabilistic adversarial training.

\subsection{Deriving the Empirical Loss Function}

\textbf{Step 1: The Idealized Objective.} We first attempt to formulate an objective by directly applying importance sampling to the theoretical lower bound.

\begin{proposition}[Idealized Objective]
Let $p_{dis}$, $p_{vic}$, and $p_{adv}$ be defined as in Section 3. Maximizing the probabilistic robustness lower bound with respect to the classifier parameters $\theta$ is mathematically equivalent to minimizing the following idealized expected loss:
\begin{equation*}
\mathcal{J}_{ideal}(\theta)=\mathbb{E}_{X\sim p_{adv}(\cdot;x_{ori},y_{ori},\theta)}\left[\frac{Z_{adv}}{p_{vic}(X;y_{ori},\theta)}c_{2}f(X,y_{ori};\theta)\right]
\end{equation*}
\end{proposition}
See Appendix~\ref{app:proof} for the proof. While Proposition 3 demonstrates that the intractable $\log Z_{vic}$ cancels out, directly optimizing $\mathcal{J}_{ideal}(\theta)$ is highly problematic. The sampling distribution $p_{adv}(X;x_{ori},y_{ori},\theta)\propto p_{dis}(X;x_{ori})p_{vic}(X;y_{ori},\theta)$ depends on the parameters $\theta$. Differentiating through this expectation is notoriously unstable and requires complex score-function estimators.

\textbf{Step 2: The Gradient-First Approach.} To circumvent the challenge of differentiating through a parameterized distribution, we propose computing the gradient of the lower bound before applying importance sampling.

\begin{proposition}[Gradient of the Objective]
\label{thm:prop4}
The exact gradient of the negative lower bound with respect to $\theta$ can be expressed as an expectation over the adversarial distribution $p_{adv}$:
\begin{equation*}
\nabla_{\theta}\mathcal{J}(\theta)=\mathbb{E}_{X\sim p_{adv}(\cdot;x_{ori},y_{ori},\theta)}\left[\frac{Z_{adv}}{p_{vic}(X;y_{ori},\theta)}\nabla_{\theta}(c_{2}f(X,y_{ori};\theta))\right]
\end{equation*}
\end{proposition}
See Appendix~\ref{app:proof} for the proof. This formulation decouples the sampling distribution from the gradient computation. Because the gradient operator $\nabla_{\theta}$ only applies to the loss term $f$, the importance weight $w(X;x_{ori},y_{ori},\theta)=\frac{Z_{adv}}{p_{vic}(X;y_{ori},\theta)}$ acts purely as a scaling constant for the gradients of individual samples.

\textbf{Step 3: Self-Normalized Importance Sampling.} Although Proposition 4 provides a relatively tractable gradient direction, evaluating the theoretical weight $w(X;x_{ori},y_{ori},\theta)$ is still impossible. By substituting $p_{vic}(X;y_{ori},\theta)=\frac{1}{Z_{vic}}\exp(c_{2}f(X,y_{ori};\theta))$, we reveal the hidden constants:
\begin{equation*}
w(X;x_{ori},y_{ori},\theta)=Z_{adv}Z_{vic}\exp(-c_{2}f(X,y_{ori};\theta))
\end{equation*}
Both $Z_{adv}$ and $Z_{vic}$ are intractable to compute. However, for a given clean image $x_{ori}$ and a specific set of model parameters $\theta$, their product is strictly a constant across different sampled perturbations.

To circumvent the intractability of these constants, we employ Self-Normalized Importance Sampling (SNIS). Given a mini-batch of adversarial examples $X^{(1)},...,X^{(B)}$ sampled from $\padv$, we normalize the weights across the batch:
\begin{align*}
\hat{w}^{(i)}&=\frac{w(X^{(i)};x_{ori},y_{ori},\theta)}{\sum_{j=1}^{B}w(X^{(j)};x_{ori},y_{ori},\theta)} =\frac{Z_{adv}Z_{vic}\exp(-c_{2}f(X^{(i)},y_{ori};\theta))}{\sum_{j=1}^{B}Z_{adv}Z_{vic}\exp(-c_{2}f(X^{(j)},y_{ori};\theta))} \\
&=\frac{\exp(-c_{2}f(X^{(i)},y_{ori};\theta))}{\sum_{j=1}^{B}\exp(-c_{2}f(X^{(j)},y_{ori};\theta))}
\end{align*}
The intractable terms $Z_{adv}$ and $Z_{vic}$ cancel out exactly. This self-normalized weighting scheme naturally corresponds to applying a Softmax function. Consequently, the true gradient derived in Proposition 4 can be estimated by backpropagating through the following empirical loss:
\begin{equation*}
\mathcal{L}_{emp}(\theta)=\sum_{i=1}^{B}\hat{w}^{(i)}f(X^{(i)},y_{ori};\theta)
\end{equation*}
where $\hat{w}^{(i)}$ are treated as fixed constants (i.e., stop-gradient) during optimization, as explicitly implemented in Lines 16-18 of Algorithm~\ref{alg:full_training} (Appendix~\ref{app:code}).

\subsection{Connection to Standard Adversarial Training}

Standard adversarial training (AT) methods, such as PGD with random restarts, can be viewed probabilistically as drawing samples from an implicit adversarial distribution $p_{adv}$ to minimize the unweighted expected loss $\mathbb{E}_{X\sim p_{adv}(\cdot | \xori, \yori)}[f(X,y_{ori};\theta)]$. Empirically, this directly corresponds to optimizing a uniformly weighted batch loss $\frac{1}{B}\sum_{i=1}^{B}f(X^{(i)},y_{ori};\theta)$.

In this work, both our idealized objective $\mathcal{J}_{ideal}$ and empirical loss $\mathcal{L}_{emp}$ recover this exact AT paradigm, differing solely by a principled scaling factor. Instead of treating all samples equally, Probabilistic Adversarial Training (PAT) optimizes a weighted loss using the self-normalized importance weights $\hat{w}^{(i)} \propto \exp(-c_2 f(X^{(i)},y_{ori};\theta))$. 

This negative exponent acts as a direct penalty on extreme adversarial examples. While standard AT strictly focuses on worst-case perturbations, PAT deliberately down-weights extreme outliers. This weighting mechanism prevents the model from over-optimizing for worst-case attacks, thereby improving overall probabilistic robustness.

\subsection{Convergence of the Empirical Gradient Estimator}
We present a convergence theorem demonstrating that our self-normalized importance sampling (SNIS) estimator is asymptotically consistent.

\begin{theorem}[Asymptotic Consistency of the SNIS Gradient Estimator]
\label{thm:snis}
Let $X^{(1)}, \dots, X^{(B)}$ be i.i.d. samples drawn from the adversarial distribution $p_{adv}(\cdot; x_{ori}, y_{ori}, \theta)$. Define the unnormalized importance weights as $\tilde{w}(X; y_{ori}, \theta) = \exp(-c_2 f(X, y_{ori}; \theta))$. The SNIS empirical gradient estimator is given by:
\begin{equation*}
    \hat{g}_B(\theta) = \sum_{i=1}^{B} \hat{w}^{(i)} \nabla_\theta f(X^{(i)}, y_{ori}; \theta) = \frac{\sum_{i=1}^{B} \tilde{w}(X^{(i)}; y_{ori}, \theta) \nabla_\theta f(X^{(i)}, y_{ori}; \theta)}{\sum_{i=1}^{B} \tilde{w}(X^{(i)}; y_{ori}, \theta)}
\end{equation*}
Assume the expected gradient under the distance distribution is finite, i.e., $\mathbb{E}_{X \sim p_{dis}(\cdot; x_{ori})}[|\nabla_\theta f(X, y_{ori}; \theta)|] < \infty$. As the batch size $B \to \infty$, the estimator $\hat{g}_B(\theta)$ converges almost surely to the true expected gradient under $p_{dis}(\cdot; x_{ori})$:
\begin{equation*}
    \hat{g}_B(\theta) \xrightarrow{a.s.} \mathbb{E}_{X \sim p_{dis}(\cdot; x_{ori})}[\nabla_\theta f(X, y_{ori}; \theta)]
\end{equation*}
\end{theorem}
The proof is provided in Appendix~\ref{app:proof}.

\subsection{Batch-Level Approximation and Implicit Regularization}
\label{sec:batch}
While Theorem 5 assumes multiple perturbations per single image $x_{ori}$, this is computationally expensive. For efficiency, our practical implementation (Algorithm~\ref{alg:full_training} in Appendix~\ref{app:code}) applies SNIS across a mini-batch of clean images $\{(x^{(i)},y^{(i)})\}\sim p_{data}$, generating only one perturbation per image. This batch-level formulation alters the weighting scheme. A strict estimator over the joint distribution requires weighting each sample by an intractable local adversarial volume:
\begin{equation*}
\overline{w}_{true}^{(i)}\propto Z_{adv}^{(i)}Z_{vic}^{(i)}\exp(-c_{2}f(X^{(i)},y^{(i)};\theta))
\end{equation*}
By definition, $Z_{adv}^{(i)}Z_{vic}^{(i)}=\mathbb{E}_{U\sim p_{dis}(\cdot;x^{(i)})}[\exp(c_{2}f(U,y^{(i)};\theta))]$. For highly vulnerable samples, this large expected volume theoretically compensates for the sharp decay of $\exp(-c_{2}f)$. Naively omitting the volume term while retaining a large attack temperature $c_{2}$ leads to \textbf{catastrophic weight collapse}: the weights of crucial hard samples are driven to near zero, forcing the Softmax to collapse onto the easiest samples and defeating the purpose of adversarial training.

To counteract this without calculating the intractable volume, we introduce a small smoothing factor $\beta$ (e.g., $\beta=0.001$) to replace $c_2$ in the final empirical weighting:
\begin{equation}
\label{eq:wemp}
\tilde{w}_{emp}^{(i)}=\exp(-\beta f(X^{(i)},y^{(i)};\theta))
\end{equation}
This $\beta$ acts as a conceptual temperature scaling that prevents weight collapse for valid hard samples. Concurrently, retaining the negative exponential structure preserves the algorithmic benefit of omitting $Z_{adv}^{(i)}Z_{vic}^{(i)}$: \textbf{implicit regularization against robust overfitting}. Outliers or mislabeled data, which theoretically possess exploding local volumes, are penalized. This extends our core philosophy of down-weighting extreme worst-case outliers to the global data manifold. The theoretical analysis of this batch-level estimator is provided in Appendix~\ref{app:batch}.


\begin{table}[t]
\centering
\caption{Comprehensive evaluation of Probabilistic Robustness (PR) and test accuracies across CIFAR-10 and CIFAR-100 datasets. PR is evaluated under escalating perturbation scales
}
\label{tab:evaluation_results}
\resizebox{\textwidth}{!}{%
\begin{tabular}{c c l c c c c c c c}
\toprule
{\textbf{D.}} & {\textbf{M.}} & \textbf{Attack} & \textbf{Acc. \%} & \textbf{PGD \%} & \textbf{CW \%} & PR($\cdot$, 0.1) \% & PR($\cdot$, 0.12) \% & PR($\cdot$, 0.15) \% & PR($\cdot$, 0.2) \% \\
\midrule
\multirow{18}{*}{\rotatebox{90}{CIFAR-10}} & \multirow{9}{*}{\rotatebox{90}{ResNet-18}} & Clean & \textbf{93.75} & 0.00 & 0.00 & 49.82 & 36.62 & 24.19 & 16.36 \\
 &  & FGSM & 87.16 & 0.01 & 0.05 & 88.88 & 84.23 & 75.83 & 59.91 \\
 &  & PGD & 82.00 & 45.80 & 44.94 & 95.26 & 92.55 & 87.72 & 77.03 \\
 &  & ALP & 81.24 & 46.65 & 45.56 & 95.36 & 92.88 & 88.16 & 77.94 \\
 &  & CLP & 83.23 & 44.11 & 43.34 & 95.37 & 93.03 & 88.47 & 78.77 \\
 &  & TRADES & 79.62 & 48.78 & 45.91 & 94.33 & 91.97 & 87.33 & 77.33 \\
 &  & MART & 78.56 & \textbf{51.30} & \textbf{47.02} & 94.26 & 91.86 & 87.36 & 78.29 \\
 &  & AT-PR & 80.33 & 40.58 & 39.23 & 95.25 & 92.80 & 88.10 & 78.22 \\
 &  & PAT & 82.44 & 32.53 & 30.51 & \textbf{95.53} & \textbf{93.43} & \textbf{89.26} & \textbf{80.62} \\
\cmidrule{2-10}
 & \multirow{9}{*}{\rotatebox{90}{WRN-50-2}} & Clean & \textbf{91.89} & 0.00 & 0.00 & 56.99 & 44.66 & 31.52 & 20.85 \\
 &  & FGSM & 87.60 & 0.00 & 0.00 & 79.79 & 67.51 & 44.35 & 21.04 \\
 &  & PGD & 76.49 & 46.54 & \textbf{44.33} & 95.59 & 93.56 & 89.44 & 79.98 \\
 &  & ALP & 72.45 & 44.42 & 41.50 & 95.71 & 93.76 & 89.71 & 81.23 \\
 &  & CLP & 79.61 & 37.71 & 36.41 & 95.86 & 93.69 & 88.94 & 78.70 \\
 &  & TRADES & 74.03 & 44.38 & 40.82 & 94.72 & 92.62 & 88.48 & 79.89 \\
 &  & MART & 69.87 & \textbf{46.83} & 42.29 & 94.93 & 92.83 & 89.26 & 81.93 \\
 &  & AT-PR & 75.93 & 40.03 & 39.21 & 95.47 & 93.61 & 89.61 & 80.23 \\
 &  & PAT & 78.26 & 31.94 & 29.91 & \textbf{95.98} & \textbf{93.96} & \textbf{90.96} & \textbf{82.19} \\
\midrule
\multirow{18}{*}{\rotatebox{90}{CIFAR-100}} & \multirow{9}{*}{\rotatebox{90}{ResNet-18}} & Clean & \textbf{74.61} & 0.00 & 0.00 & 27.45 & 19.20 & 11.93 & 7.08 \\
 &  & FGSM & 59.56 & 0.00 & 0.00 & 64.50 & 53.45 & 38.65 & 21.62 \\
 &  & PGD & 54.50 & 22.48 & 21.18 & 89.27 & 83.31 & 71.01 & 46.90 \\
 &  & ALP & 54.88 & 23.83 & 21.98 & 90.45 & 84.69 & 72.95 & 49.35 \\
 &  & CLP & 55.71 & 22.08 & 19.28 & 90.05 & 84.80 & 74.50 & 54.53 \\
 &  & TRADES & 53.56 & 26.74 & 22.66 & 89.07 & 84.23 & 74.74 & 56.24 \\
 &  & MART & 52.13 & \textbf{26.88} & \textbf{23.51} & 90.02 & 84.86 & 74.66 & 51.76 \\
 &  & AT-PR & 52.67 & 21.48 & 20.28 & 89.09 & 83.14 & 71.18 & 47.64 \\
 &  & PAT & 53.54 & 17.50 & 14.87 & \textbf{90.77} & \textbf{86.58} & \textbf{78.07} & \textbf{59.78} \\
\cmidrule{2-10}
 & \multirow{9}{*}{\rotatebox{90}{WRN-50-2}} & Clean & \textbf{70.41} & 0.00 & 0.00 & 25.51 & 17.95 & 11.23 & 6.24 \\
 &  & FGSM & 51.04 & 0.00 & 0.00 & 39.19 & 25.26 & 13.12 & 5.26 \\
 &  & PGD & 49.55 & 24.04 & \textbf{22.00} & 92.06 & 87.94 & 78.85 & 58.99 \\
 &  & ALP & 44.38 & 23.10 & 20.80 & 92.07 & 88.07 & 79.92 & 62.17 \\
 &  & CLP & 52.95 & 19.15 & 16.94 & 90.91 & 86.36 & 76.20 & 56.99 \\
 &  & TRADES & 49.85 & 23.01 & 19.18 & 90.64 & 86.46 & 78.32 & 62.03 \\
 &  & MART & 42.61 & \textbf{24.67} & 21.15 & 92.03 & 88.52 & 81.41 & 66.51 \\
 &  & AT-PR & 48.47 & 22.02 & 20.98 & 92.11 & 88.08 & 79.87 & 61.58 \\
 &  & PAT & 49.95 & 17.26 & 15.98 & \textbf{92.28} & \textbf{89.49} & \textbf{82.18} & \textbf{67.12} \\
\bottomrule
\end{tabular}
}
\end{table}

\section{Experiments}
\label{sec:exps}
Due to space constraints, this section focuses on the main evaluation results. Comprehensive details regarding the experimental settings, datasets, and hyperparameters are deferred to Appendix~\ref{app:details}.

\subsection{Probabilistic Robustness}

Table~\ref{tab:evaluation_results} summarizes the performance of PAT and baselines (PGD \citep{madry2017towards}, ALP \citep{kannan2018adversarial}, CLP \citep{kannan2018adversarial}, TRADES \citep{zhang2019theoretically}, MART \citep{wang2019improving} and AT-PR \citep{zhang2025adversarial}) on CIFAR-10 and CIFAR-100 \citep{krizhevsky2009learning}. PAT consistently achieves superior Probabilistic Robustness (PR) across all evaluated architectures (WRN-50-2, ResNet-18) and datasets, with the advantage becoming more pronounced at larger perturbation scales.

While PAT shows a nominal decrease in strictly worst-case robustness (PGD/CW Acc.) compared to unweighted AT, this directly aligns with our theory. Unlike standard AT which over-fits to anomalous worst-case boundaries, PAT's theoretically derived importance weight naturally penalizes extreme outliers. This optimizes overall distributional security, effectively pushing the distance and victim distributions apart.

\subsection{Ablation Study}
\label{sec:ablation}
Table~\ref{tab:ablation} isolates the impact of our theoretically derived importance weight through two variants: \textbf{PAT(WOS)} (Langevin sampling without the importance weight) and \textbf{PGD-COR/FGSM-COR} (standard attacks equipped with our theoretical scaling factor).

PAT consistently outperforms PAT(WOS). On CIFAR-100 (WRN-50-2), removing the importance weight drops PR(0.2) from $67.12\%$ to $63.98\%$. This confirms that sampling from $p_{adv}$ alone is insufficient; down-weighting extreme outliers via our theoretical formulation is an essential requirement for maximizing probabilistic robustness.

Our importance weight acts as a versatile plug-in regularizer. Applying it to standard PGD (PGD-COR) improves PR(0.2) on CIFAR-100 from $58.99\%$ to $62.21\%$. Notably, the performance of FGSM-COR exhibits more irregularity. This discrepancy is theoretically expected: as a deterministic, single-step attack, FGSM lacks the iterative exploration and stochasticity required for the importance weight to effectively evaluate the local adversarial distribution. Nonetheless, the consistent gains in PGD-COR demonstrate the broader utility of our weighting scheme in improving the distributional robustness of traditional iterative defense methods.

\begin{table}[tbp]
\centering
\caption{Ablation results for the importance weight across various architectures. PAT(WOS) removes the weighting factor; COR denotes the application of our scaling factor to standard baselines.}
\label{tab:ablation}
\resizebox{\textwidth}{!}{%
\begin{tabular}{c c l c c c c c c c}
\toprule
{\textbf{D.}} & {\textbf{M.}} & \textbf{Attack} & \textbf{Acc. \%} & \textbf{PGD \%} & \textbf{CW \%} & PR($\cdot$, 0.1) \% & PR($\cdot$, 0.12) \% & PR($\cdot$, 0.15) \% & PR($\cdot$, 0.2) \% \\
\midrule
\multirow{12}{*}{\rotatebox{90}{CIFAR-10}} & \multirow{6}{*}{\rotatebox{90}{ResNet-18}} & FGSM & 87.16 & 0.01 & \textbf{0.05} & 88.88 & 84.23 & \textbf{75.83} & \textbf{59.91} \\
 &  & FGSM-COR & \textbf{89.04} & \textbf{0.02} & 0.03 & \textbf{89.03} & \textbf{84.53} & 75.33 & 56.45 \\
\cmidrule{3-10}
 &  & PGD & \textbf{82.00} & 45.80 & \textbf{44.94} & \textbf{95.26} & \textbf{92.55} & \textbf{87.72} & \textbf{77.03} \\
 &  & PGD-COR & 81.75 & \textbf{45.81} & 44.72 & 95.03 & 92.38 & 87.33 & 75.97 \\
\cmidrule{3-10}
 &  & PAT(WOS) & \textbf{82.96} & 32.15 & 30.34 & 94.86 & 92.56 & 88.36 & 79.81 \\
 &  & PAT & 82.44 & \textbf{32.53} & \textbf{30.51} & \textbf{95.53} & \textbf{93.43} & \textbf{89.26} & \textbf{80.62} \\
\cmidrule{3-10}
 & \multirow{6}{*}{\rotatebox{90}{WRN-50-2}} & FGSM & \textbf{87.60} & 0.00 & 0.00 & 79.79 & 67.51 & 44.35 & 21.04 \\
 &  & FGSM-COR & 76.46 & \textbf{42.63} & \textbf{41.01} & \textbf{93.42} & \textbf{90.26} & \textbf{84.74} & \textbf{74.40} \\
\cmidrule{3-10}
 &  & PGD & \textbf{76.49} & \textbf{46.54} & \textbf{44.33} & 95.59 & 93.56 & 89.44 & 79.98 \\
 &  & PGD-COR & 73.45 & 45.11 & 42.71 & \textbf{95.79} & \textbf{93.85} & \textbf{89.79} & \textbf{80.62} \\
\cmidrule{3-10}
 &  & PAT(WOS) & \textbf{79.06} & \textbf{31.99} & 29.63 & 94.63 & 92.29 & 87.91 & 79.25 \\
 &  & PAT & 78.26 & 31.94 & \textbf{29.91} & \textbf{95.98} & \textbf{93.96} & \textbf{90.96} & \textbf{82.19} \\
\midrule
\multirow{12}{*}{\rotatebox{90}{CIFAR-100}} & \multirow{6}{*}{\rotatebox{90}{ResNet-18}} & FGSM & \textbf{59.56} & \textbf{0.00} & 0.00 & \textbf{64.50} & \textbf{53.45} & \textbf{38.65} & \textbf{21.62} \\
 &  & FGSM-COR & 56.20 & \textbf{0.00} & \textbf{0.01} & 62.68 & 51.06 & 35.49 & 18.52 \\
\cmidrule{3-10}
 &  & PGD & 54.50 & 22.48 & 21.18 & 89.27 & 83.31 & 71.01 & 46.90 \\
 &  & PGD-COR & \textbf{55.02} & \textbf{23.42} & \textbf{22.06} & \textbf{90.11} & \textbf{84.23} & \textbf{72.39} & \textbf{47.92} \\
\cmidrule{3-10}
 &  & PAT(WOS) & 53.21 & 16.84 & 14.53 & 90.47 & 86.24 & \textbf{78.07} & \textbf{59.83} \\
 &  & PAT & \textbf{53.54} & \textbf{17.50} & \textbf{14.87} & \textbf{90.77} & \textbf{86.58} & \textbf{78.07} & 59.78 \\
\cmidrule{3-10}
 & \multirow{6}{*}{\rotatebox{90}{WRN-50-2}} & FGSM & 51.04 & \textbf{0.00} & \textbf{0.00} & 39.19 & 25.26 & \textbf{13.12} & \textbf{5.26} \\
 &  & FGSM-COR & \textbf{52.50} & \textbf{0.00} & \textbf{0.00} & \textbf{41.66} & \textbf{26.31} & 11.75 & 3.88 \\
\cmidrule{3-10}
 &  & PGD & \textbf{49.55} & \textbf{24.04} & \textbf{22.00} & 92.06 & 87.94 & 78.85 & 58.99 \\
 &  & PGD-COR & 44.87 & 22.74 & 20.46 & \textbf{92.66} & \textbf{88.79} & \textbf{80.55} & \textbf{62.21} \\
\cmidrule{3-10}
 &  & PAT(WOS) & \textbf{50.17} & 16.92 & 14.96 & 91.65 & 87.91 & 80.22 & 63.98 \\
 &  & PAT & 49.95 & \textbf{17.26} & \textbf{15.98} & \textbf{92.28} & \textbf{89.49} & \textbf{82.18} & \textbf{67.12} \\
\bottomrule
\end{tabular}
}
\end{table}

\section{Related Work}

\textbf{Adversarial Training for Worst-Case Robustness.} Traditional adversarial training (AT) formulates a min-max optimization problem to defend against worst-case perturbations \citep{madry2017towards}. A vast literature has developed around improving this worst-case robustness (WCR), including regularization techniques like TRADES \citep{zhang2019theoretically} and MART \citep{wang2019improving}, logit pairing methods such as ALP and CLP \citep{kannan2018adversarial}, and efficient training variants like Fast-AT \citep{wong2020fast} and Free-AT \citep{shafahi2019adversarial}. Complementary to our work, \citet{li2024adversarial} theoretically analyzed the feature learning dynamics of AT. However, these WCR-focused methods inherently optimize for extreme, deterministic adversarial outliers, which frequently compromises standard accuracy and broader distributional security \citep{zhao2026probabilistic}.

\textbf{Probabilistic Robustness Assessment.} To bridge the gap between extreme-case vulnerability and practical reliability, \citet{webb2018statistical} formally defined Probabilistic Robustness (PR) to evaluate the overall likelihood of adversarial examples within a local neighborhood. Following this, significant progress has been made in PR assessment, including scalable statistical estimators \citep{tit2021efficient, karim2023gradient}, probabilistic verification approaches \citep{weng2019proven}, and extensions to functional perturbations \citep{zhang2022proa}. A comprehensive survey of these evaluation metrics is provided by Zhao \citep{zhao2026probabilistic}. While evaluating PR has become increasingly sophisticated, dedicated training frameworks designed to directly optimize it remain sparse.

\textbf{Defenses Targeting Probabilistic Robustness.} Only a few recent works have explicitly aimed at improving PR during training. \citet{wang2021statistically} demonstrated that adding random perturbations enhances PR, though it yields near-zero worst-case robustness. Other approaches \citep{robey2022probabilistically, zhang2024prass} proposed training with risk measures such as Conditional Value-at-Risk (CVaR) or Entropic Value-at-Risk (EVaR), but also struggled to maintain competitive standard adversarial robustness. The most direct comparison to our work is the recent AT-PR method \citep{zhang2025adversarial}, which explicitly alters the inner maximization of AT to target PR. However, their approach relies on a heuristic, gradient-based boundary search to empirically find the ``widest peak'' in the loss landscape. In contrast, our Probabilistic Adversarial Training (PAT) establishes a rigorous theoretical foundation: we formalize untargeted probabilistic attacks via Langevin Dynamics, prove a KL-based lower bound for PR, and employ Self-Normalized Importance Sampling (SNIS) to derive a fully tractable, principled optimization objective.

\section{Conclusion}
In this work, we provided a unified probabilistic perspective on adversarial defense by proving that $KL(p_{dis}||p_{vic}) - \log Z_{vic}$ serves as a lower bound for probabilistic robustness (PR). By maximizing this tractable surrogate objective via SNIS applied to gradient, we introduced Probabilistic Adversarial Training (PAT). Our theoretical framework not only offers a principled probabilistic interpretation of standard adversarial training but also yields a natural importance-weighting mechanism that penalizes extreme worst-case outliers. Empirically, PAT consistently enhances PR across various architectures. We hope this probabilistic perspective of adversarial vulnerabilities inspires more robust and theoretically grounded defense strategies in the future.

\subsubsection*{Acknowledgments}
The work presented in this paper has been supported by UKRI Future Leaders Fellowship (Grant MR/S035176/1) and funded by the European Union under the Horizon Europe project AIGGREGATE (AI-enhanced collective intelligence for resilient, ethical and user-centric awareness and decision making in CCAM applications, Grant Agreement No. 101202457). Views and opinions expressed are those of the author(s) only and do not necessarily reflect those of the European Union. Neither the European Union nor the granting authority can be held responsible for them.


\bibliography{iclr2027_conference}
\bibliographystyle{iclr2027_conference}

\appendix
\clearpage
\section*{Appendix}

\setcounter{theorem}{0}

\section{Pseudo Code of the Algorithms}
\label{app:code}

\begin{algorithm}[h]
   \caption{Untargeted Probabilistic Adversarial Attack}
   \label{alg:untargetedatk}
\begin{algorithmic}
   \STATE {\bfseries Input:} Original image \(\xori\in [0,1]^d\), original class \(\yori\), cross-entropy loss \(f\) corresponding to the victim classifier, parameters \(c_1,c_2>0\), step size \(\eta\), noise scale \(\sigma\), gradient clipping threshold \(\rho>0\), total timesteps \(T\).
   \STATE {\bfseries Output:} Adversarial example \(\xadv\).
   \STATE \(\xadv\sim \text{Uniform}[0,1]^d\) \hfill $\triangleright$ Initialize
   \FOR{\(t=1\) {\bfseries to} \(T\)}
      \STATE \(z_t\sim\mathcal{N}(0,I)\) \hfill $\triangleright$ Generate noise for Langevin Dynamics
      \STATE \(\xadv \leftarrow \Pi_{[0,1]^d}\left(\xadv+\sigma z_t\right)\)\hfill $\triangleright$ Apply noise
      \STATE \(\mathcal{E} \leftarrow c_1\,\|\xadv-\xori\|_2^2-c_2\,f(\xadv,\yori)\)\hfill $\triangleright$ Calculate energy
      \STATE \(v_t \leftarrow \mathrm{clip}(\nabla_{\xadv}\mathcal{E},-\rho,\rho)\)\hfill $\triangleright$ Compute and clip gradient
      \STATE \(\xadv \leftarrow \Pi_{[0,1]^d}\left(\xadv-\eta v_t\right)\)\hfill $\triangleright$ update $\xadv$
   \ENDFOR
   \STATE \textbf{return} \(\xadv\).
\end{algorithmic}
\end{algorithm}

\begin{algorithm}[htbp]
   \caption{Probabilistic Adversarial Training}
   \label{alg:full_training}
\begin{algorithmic}[1]
   \STATE {\bfseries Input:} Training dataset $\mathcal{D}$, initial model parameters $\theta$, cross-entropy loss $f$, learning rate $\alpha$.
   \STATE {\bfseries Hyperparameters:} Total epochs $E_{\text{total}}$, batch size $B$, inner steps $T$, step size $\eta$, noise scale $\sigma$, gradient clip $\rho$, energy weights $c_1, c_2$, smoothing factor \(\beta\).
   \STATE {\bfseries Output:} Robust model parameters $\theta$.
   
   \FOR{epoch $e = 0$ {\bfseries to} $E_{\text{total}} - 1$}
      \FOR{{\bfseries each} minibatch $(\mathbf{X}, \mathbf{Y}) \sim \mathcal{D}$ with size $B$}
         
         \STATE $\triangleright$ \textbf{Phase 1: Probabilistic Adversarial Example Generation}
         \STATE $\mathbf{X}_{\text{adv}} \sim \text{Uniform}[0,1]^{B \times d}$\hfill $\triangleright$ Initialize
         \FOR{$t=1$ {\bfseries to} $T$}
            \STATE $\mathbf{Z}_t \sim \mathcal{N}(\mathbf{0},\mathbf{I})^{B \times d}$\hfill $\triangleright$ Sample noise for Langevin Dynamics
            \STATE $\mathbf{X}_{\text{adv}} \leftarrow \Pi_{[0,1]^{B \times d}}\left(\mathbf{X}_{\text{adv}} + \sigma \mathbf{Z}_t\right)$\hfill $\triangleright$ Apply noise
            \STATE $\mathcal{E} \leftarrow \sum_{i=1}^B \Big( c_1\,\|\mathbf{X}_{\text{adv}}^{(i)} - \mathbf{X}^{(i)}\|_2^2 - c_2\,f_\theta(\mathbf{X}_{\text{adv}}^{(i)}, \mathbf{Y}^{(i)}) \Big)$ \hfill $\triangleright$ Calculate batch energy
            \STATE $\mathbf{V}_t \leftarrow \mathrm{clip}(\nabla_{\mathbf{X}_{\text{adv}}}\mathcal{E}, -\rho, \rho)$ \hfill $\triangleright$ Compute and clip gradient
            \STATE $\mathbf{X}_{\text{adv}} \leftarrow \Pi_{[0,1]^{B \times d}}\left(\mathbf{X}_{\text{adv}} - \eta \mathbf{V}_t\right)$ \hfill $\triangleright$ Update $\mathbf{X}_{\text{adv}}$
         \ENDFOR
         
         \STATE $\triangleright$ \textbf{Phase 2: Model Parameter Update via Gradient Descent}
         \STATE $\boldsymbol{\ell}_i \leftarrow f_\theta(\mathbf{X}_{\text{adv}}^{(i)}, \mathbf{Y}^{(i)})$\hfill $\triangleright$ Compute sample losses $\boldsymbol{\ell} \in \mathbb{R}^B$

            \STATE $\tilde{\mathbf{w}}_{\text{pvic}} \leftarrow \mathrm{Detach}(-\beta\, \boldsymbol{\ell})$\hfill $\triangleright$ Apply smoothing factor and stop gradient
            \STATE $\hat{\mathbf{w}}_{\text{pvic}} \leftarrow \mathrm{Softmax}(\tilde{\mathbf{w}}_{\text{pvic}})$ \hfill $\triangleright$ Softmax over batch dimension
            \STATE $\mathcal{L} \leftarrow \boldsymbol{\ell}^\top \hat{\mathbf{w}}_{\text{pvic}}$ \hfill $\triangleright$ Final weighted loss

         \STATE $\theta \leftarrow \theta - \alpha \nabla_\theta \mathcal{L}$ \hfill $\triangleright$ Update model parameters
      \ENDFOR
   \ENDFOR
   \STATE \textbf{return} $\theta$.
\end{algorithmic}
\end{algorithm}

\section{Proofs of Theorems in the Main Text}
\label{app:proof}
\begin{proposition}[Well-definedness of the untargeted victim distribution]
Let \(K=[0,1]^d\) and let \(f(\cdot,y_{\mathrm{ori}}):K\to\mathbb{R}\) be continuous. For any \(c>0\), define
\[
    p_{\mathrm{vic}}(x;y_{\mathrm{ori}})
    =
    \frac{\exp\big(c f(x,y_{\mathrm{ori}})\big)}
    {\int_K \exp\big(c f(u,y_{\mathrm{ori}})\big)\,du},
    \qquad x\in K.
\]
Then \(p_{\mathrm{vic}}(\cdot;y_{\mathrm{ori}})\) is a well-defined probability density on \(K\).
\end{proposition}

\begin{proof}
Since \(K=[0,1]^d\) is compact and \(f(\cdot,y_{\mathrm{ori}})\) is continuous, 
\(\exp(c f(\cdot,y_{\mathrm{ori}}))\) is also continuous on \(K\). Hence it is bounded and measurable. 
Moreover, it is strictly positive on \(K\), so the normalizing constant
\[
    Z_{\mathrm{vic}}
    =
    \int_K \exp\big(c f(u,y_{\mathrm{ori}})\big)\,du
\]
satisfies \(0<Z_{\mathrm{vic}}<\infty\). Therefore \(p_{\mathrm{vic}}\) is non-negative, measurable, and integrates to one over \(K\). Thus it is a well-defined probability density.
\end{proof}

\begin{theorem}[KL-Based Lower Bound on Probabilistic Robustness]
Let $p_{\mathrm{dis}}(\cdot; x_{\mathrm{ori}})$ be a probability density on $[0,1]^{d}$. Suppose the victim distribution parameterized by $\theta$ is defined as
\begin{equation*}
    p_{\mathrm{vic}}(x; y_{\mathrm{ori}}, \theta) = \frac{\exp(c_2 f(x, y_{\mathrm{ori}}; \theta))}{Z_{\mathrm{vic}}},
\end{equation*}
where $Z_{\mathrm{vic}} = \int_{\mathcal{X}} \exp(c_2 f(u, y_{\mathrm{ori}}; \theta)) du$ is the normalizing constant. Let $\gamma_{*}$ be defined as above and assume $\gamma_{*}>0$. Then, the probabilistic robustness satisfies:
\begin{equation*}
    \mathrm{PR}(p_{\mathrm{dis}}, s) \ge 1 - \frac{\log Z_{\mathrm{vic}} - H(p_{\mathrm{dis}}) - \mathrm{KL}(p_{\mathrm{dis}} \| p_{\mathrm{vic}})}{\gamma_{*}},
\end{equation*}
where $H(p_{\mathrm{dis}}) = -\int_{\mathcal{X}} p_{\mathrm{dis}}(x) \log p_{\mathrm{dis}}(x) dx$ is the differential entropy of the distance distribution.
\end{theorem}

\begin{proof}
By the definition of $\gamma_*$, for all $x \in E_{adv}$, we have $c_2 f(x, y_{ori}) \ge \gamma_*$. Thus, we can lower bound the expected loss under $p_{dis}$:
\begin{equation*}
\mathbb{E}_{X \sim p_{dis}}[c_2 f(X, y_{ori})] \ge \int_{E_{adv}} p_{dis}(x) c_2 f(x, y_{ori}) dx \ge \gamma_* \int_{E_{adv}} p_{dis}(x) dx = \gamma_* \mathcal{I}[p_{dis}, s]
\end{equation*}

On the other hand, expanding the KL divergence with $p_{vic}(x; y_{ori}) = \frac{\exp(c_2 f(x, y_{ori}))}{Z_{vic}}$ yields:
\begin{equation*}
\begin{aligned}
KL(p_{dis}||p_{vic}) &= \int p_{dis}(x) \log p_{dis}(x) dx - \int p_{dis}(x) \log p_{vic}(x; y_{ori}) dx \\
&= -H(p_{dis}) - \mathbb{E}_{X \sim p_{dis}}[c_2 f(X, y_{ori}) - \log Z_{vic}] \\
&= -H(p_{dis}) + \log Z_{vic} - \mathbb{E}_{X \sim p_{dis}}[c_2 f(X, y_{ori})] \\
&\le -H(p_{dis}) + \log Z_{vic} - \gamma_* \mathcal{I}[p_{dis}, s]
\end{aligned}
\end{equation*}

Rearranging the inequality gives:
\begin{equation*}
\mathcal{I}[p_{dis}, s] \le \frac{\log Z_{vic} - H(p_{dis}) - KL(p_{dis}||p_{vic})}{\gamma_*}
\end{equation*}

Finally, by the definition of probabilistic robustness $PR(p_{dis}, s) = 1 - \mathcal{I}[p_{dis}, s]$, we conclude:
\begin{equation*}
PR(p_{dis}, s) \ge 1 - \frac{\log Z_{vic} - H(p_{dis}) - KL(p_{dis}||p_{vic})}{\gamma_*}
\end{equation*}
\end{proof}

\begin{proposition}[Idealized Objective]
Let $p_{dis}$, $p_{vic}$, and $p_{adv}$ be defined as in Section 3. Maximizing the probabilistic robustness lower bound with respect to the classifier parameters $\theta$ is mathematically equivalent to minimizing the following idealized expected loss:
\begin{equation*}
\mathcal{J}_{ideal}(\theta)=\mathbb{E}_{X\sim p_{adv}(\cdot;x_{ori},y_{ori},\theta)}\left[\frac{Z_{adv}}{p_{vic}(X;y_{ori},\theta)}c_{2}f(X,y_{ori};\theta)\right]
\end{equation*}
\end{proposition}

\begin{proof}

By expanding the KL divergence and changing the measure from $p_{dis}$ to $p_{adv}$, we have:
\begin{align*}
&KL(p_{dis}(\cdot;x_{ori})||p_{vic}(\cdot;y_{ori},\theta))-\log Z_{vic} \\
&=\mathbb{E}_{X\sim p_{dis}(\cdot;x_{ori})}[\log p_{dis}(X;x_{ori})]-\int p_{dis}(x;x_{ori})\log p_{vic}(x;y_{ori},\theta)dx-\log Z_{vic} \\
&=-H(p_{dis})-\int p_{adv}(x;x_{ori},y_{ori},\theta)\frac{Z_{adv}p_{dis}(x;x_{ori})}{p_{dis}(x;x_{ori})p_{vic}(x;y_{ori},\theta)}\log p_{vic}(x;y_{ori},\theta)dx-\log Z_{vic} \\
&=-H(p_{dis})-\mathbb{E}_{X\sim p_{adv}(\cdot;x_{ori},y_{ori},\theta)}\left[\frac{Z_{adv}}{p_{vic}(X;y_{ori},\theta)}(c_{2}f(X,y_{ori};\theta)-\log Z_{vic})\right]-\log Z_{vic} \\
&=-H(p_{dis})-\mathbb{E}_{X\sim p_{adv}(\cdot;x_{ori},y_{ori},\theta)}\left[\frac{Z_{adv}}{p_{vic}(X;y_{ori},\theta)}c_{2}f(X,y_{ori};\theta)\right] \\
&\quad +\log Z_{vic}\underbrace{\mathbb{E}_{X\sim p_{adv}(\cdot;x_{ori},y_{ori},\theta)}\left[\frac{Z_{adv}}{p_{vic}(X;y_{ori},\theta)}\right]}_{=1}-\log Z_{vic} \\
&=-H(p_{dis})-\mathbb{E}_{X\sim p_{adv}(\cdot;x_{ori},y_{ori},\theta)}\left[\frac{Z_{adv}}{p_{vic}(X;y_{ori},\theta)}c_{2}f(X,y_{ori};\theta)\right]
\end{align*}
Since $H(p_{dis})$ is constant with respect to $\theta$, maximizing the lower bound is equivalent to minimizing $\mathcal{J}_{ideal}(\theta)$.
\end{proof}

\begin{proposition}[Gradient of the Objective]
The exact gradient of the negative lower bound with respect to $\theta$ can be expressed as an expectation over the adversarial distribution $p_{adv}$:
\begin{equation*}
\nabla_{\theta}\mathcal{J}(\theta)=\mathbb{E}_{X\sim p_{adv}(\cdot;x_{ori},y_{ori},\theta)}\left[\frac{Z_{adv}}{p_{vic}(X;y_{ori},\theta)}\nabla_{\theta}(c_{2}f(X,y_{ori};\theta))\right]
\end{equation*}
\end{proposition}
\begin{proof}
We return to the negative lower bound before importance sampling: 
\[-KL(p_{dis}||p_{vic})+\log Z_{vic}=H(p_{dis})+\mathbb{E}_{X\sim p_{dis}(\cdot;x_{ori})}[c_{2}f(X,y_{ori};\theta)].\]
Since $p_{dis}(\cdot;x_{ori})$ depends only on the clean image $x_{ori}$ and is entirely independent of $\theta$, the gradient operator can pass directly inside the expectation:
\begin{align*}
&\nabla_{\theta}\left(\mathbb{E}_{X\sim p_{dis}(\cdot;x_{ori})}[c_{2}f(X,y_{ori};\theta)]\right)\\
=&\mathbb{E}_{X\sim p_{dis}(\cdot;x_{ori})}[\nabla_{\theta}c_{2}f(X,y_{ori};\theta)] \\
=&\int p_{adv}(x;x_{ori},y_{ori},\theta)\frac{p_{dis}(x;x_{ori})}{p_{adv}(x;x_{ori},y_{ori},\theta)}\nabla_{\theta}c_{2}f(x,y_{ori};\theta)dx \\
=&\mathbb{E}_{X\sim p_{adv}(\cdot;x_{ori},y_{ori},\theta)}\left[\frac{Z_{adv}}{p_{vic}(X;y_{ori},\theta)}\nabla_{\theta}c_{2}f(X,y_{ori};\theta)\right]
\end{align*}
\end{proof}

\begin{theorem}[Asymptotic Consistency of the SNIS Gradient Estimator]
Let $X^{(1)}, \dots, X^{(B)}$ be i.i.d. samples drawn from the adversarial distribution $p_{adv}(\cdot; x_{ori}, y_{ori}, \theta)$. Define the unnormalized importance weights as $\tilde{w}(X; y_{ori}, \theta) = \exp(-c_2 f(X, y_{ori}; \theta))$. The SNIS empirical gradient estimator is given by:
\begin{equation*}
    \hat{g}_B(\theta) = \sum_{i=1}^{B} \hat{w}^{(i)} \nabla_\theta f(X^{(i)}, y_{ori}; \theta) = \frac{\sum_{i=1}^{B} \tilde{w}(X^{(i)}; y_{ori}, \theta) \nabla_\theta f(X^{(i)}, y_{ori}; \theta)}{\sum_{i=1}^{B} \tilde{w}(X^{(i)}; y_{ori}, \theta)}
\end{equation*}
Assume the expected gradient under the distance distribution is finite, i.e., $\mathbb{E}_{X \sim p_{dis}(\cdot; x_{ori})}[|\nabla_\theta f(X, y_{ori}; \theta)|] < \infty$. As the batch size $B \to \infty$, the estimator $\hat{g}_B(\theta)$ converges almost surely to the true expected gradient under $p_{dis}(\cdot; x_{ori})$:
\begin{equation*}
    \hat{g}_B(\theta) \xrightarrow{a.s.} \mathbb{E}_{X \sim p_{dis}(\cdot; x_{ori})}[\nabla_\theta f(X, y_{ori}; \theta)]
\end{equation*}
\end{theorem}

\begin{proof}
By the Strong Law of Large Numbers (SLLN), as $B \to \infty$, the sample averages in the numerator and denominator of $\hat{g}_B(\theta)$ converge almost surely to their respective expectations under $p_{adv}(\cdot; x_{ori}, y_{ori}, \theta)$:
\begin{equation*}
\begin{aligned}
    \frac{1}{B} \sum_{i=1}^{B} \tilde{w}(X^{(i)}; y_{ori}, \theta) \nabla_\theta f(X^{(i)}, y_{ori}; \theta) &\xrightarrow{a.s.} \mathbb{E}_{X \sim p_{adv}} \left[ \tilde{w}(X; y_{ori}, \theta) \nabla_\theta f(X, y_{ori}; \theta) \right] \\
    \frac{1}{B} \sum_{i=1}^{B} \tilde{w}(X^{(i)}; y_{ori}, \theta) &\xrightarrow{a.s.} \mathbb{E}_{X \sim p_{adv}} \left[ \tilde{w}(X; y_{ori}, \theta) \right]
\end{aligned}
\end{equation*}

Recall the definitions from Section 3. We can express the product of the sampling density and the unnormalized weight as:
\begin{equation*}
\begin{aligned}
    p_{adv}(X; x_{ori}, y_{ori}, \theta) \tilde{w}(X; y_{ori}, \theta) &= \frac{p_{dis}(X; x_{ori}) \exp(c_2 f(X, y_{ori}; \theta))}{Z_{adv} Z_{vic}} \exp(-c_2 f(X, y_{ori}; \theta)) \\
    &= \frac{p_{dis}(X; x_{ori})}{Z_{adv} Z_{vic}}
\end{aligned}
\end{equation*}

Substituting this identity into the expectation of the numerator yields:
\begin{equation*}
\begin{aligned}
    \mathbb{E}_{X \sim p_{adv}} \left[ \tilde{w}(X; y_{ori}, \theta) \nabla_\theta f(X, y_{ori}; \theta) \right] &= \int \frac{p_{dis}(x; x_{ori})}{Z_{adv} Z_{vic}} \nabla_\theta f(x, y_{ori}; \theta) dx \\
    &= \frac{1}{Z_{adv} Z_{vic}} \mathbb{E}_{X \sim p_{dis}} \left[ \nabla_\theta f(X, y_{ori}; \theta) \right]
\end{aligned}
\end{equation*}

Applying the same substitution to the denominator:
\begin{equation*}
    \mathbb{E}_{X \sim p_{adv}} \left[ \tilde{w}(X; y_{ori}, \theta) \right] = \int \frac{p_{dis}(x; x_{ori})}{Z_{adv} Z_{vic}} dx = \frac{1}{Z_{adv} Z_{vic}}
\end{equation*}

By the Continuous Mapping Theorem (or Slutsky's Theorem), the limit of the ratio is the ratio of their limits. Notice that the intractable constant term perfectly cancels out:
\begin{equation*}
    \hat{g}_B(\theta) \xrightarrow{a.s.} \frac{\frac{1}{Z_{adv} Z_{vic}} \mathbb{E}_{X \sim p_{dis}} \left[ \nabla_\theta f(X, y_{ori}; \theta) \right]}{\frac{1}{Z_{adv} Z_{vic}}} = \mathbb{E}_{X \sim p_{dis}(\cdot; x_{ori})} \left[ \nabla_\theta f(X, y_{ori}; \theta) \right]
\end{equation*}
This confirms that the self-normalized gradient is asymptotically unbiased and converges to the true gradient of our theoretical objective.
\end{proof}

\section{Theoretical Analysis of the Batch-Level Estimator}
\label{app:batch}
Theorem~\ref{thm:snis} establishes consistency of the SNIS gradient estimator in an idealized regime: a single clean image $x_{\mathrm{ori}}$
with label $y_{\mathrm{ori}}$, a batch of $B$ perturbations from its adversarial distribution $p_{\mathrm{adv}}(\cdot;x_{\mathrm{ori}},y_{\mathrm{ori}},\theta)$.
Algorithm~\ref{alg:full_training} differs in two respects: it draws a mini-batch of distinct clean images with one perturbation each, and it uses the smoothing factor $\beta$ in place of $c_2$ in the empirical weights (Eq.~(\ref{eq:wemp})).
Sec.~\ref{sec:batch} observed that restoring exactness of the joint estimator would require the intractable volume weights $w_{\mathrm{true}}^{(i)}$.
This section analyses the estimator Algorithm~2 actually computes and proves:
\begin{itemize}
    \item it is consistent, with an explicitly characterized limit (Theorem~\ref{thm:thm7});
    \item the limit is the exact gradient of an entropic-risk objective under a vulnerability-based reweighting of the data (Proposition~\ref{thm:prop8}, Corollary~\ref{thm:corollary9});
    \item this objective still lower-bounds probabilistic robustness for every admissible $\beta$, with Theorem~\ref{thm:kl_bound} recovered as the tightest member of the resulting family (Proposition~\ref{thm:prop10}, Corollary~\ref{thm:corollary11});
\end{itemize}

\noindent
\noindent\textbf{Setup.}
Let $\mathcal{X}=[0,1]^d$.
For each $y\in\mathcal{Y}$, the loss $f(\cdot,y;\theta)$
is continuous on the compact set $\mathcal{X}$,
as in Proposition~1, and
\[
0
\leq f(u,y;\theta)
\leq F_{\max}(\theta)
:=
\max_{y\in\mathcal{Y}}
\max_{u\in\mathcal{X}}
f(u,y;\theta)
<\infty.
\]
For a clean pair $(x,y)\in\mathcal{X}\times\mathcal{Y}$
and $t\geq 0$, let
\[
M_t(x,y;\theta)
:=
\mathbb{E}_{U\sim p_{\mathrm{dis}}(\cdot;x)}
\left[
    \exp\bigl(t f(U,y;\theta)\bigr)
\right].
\]
With $\beta\in[0,c_2]$ and $\lambda:=c_2-\beta$,
define the tilted distribution, volume ratio,
and reweighted data distribution:
\begin{align*}
p_\lambda(u;x,y,\theta)
&:=
\frac{
    p_{\mathrm{dis}}(u;x)
    \exp\bigl(\lambda f(u,y;\theta)\bigr)
}{
    M_\lambda(x,y;\theta)
},
\\
r(x,y;\theta)
&:=
\frac{
    M_\lambda(x,y;\theta)
}{
    M_{c_2}(x,y;\theta)
},
\\
\tilde p(x,y;\theta)
&\propto
p_{\mathrm{data}}(x,y)\,r(x,y;\theta).
\end{align*}
The last is well-defined since
\[
\mathbb{E}_{(x,y)\sim p_{\mathrm{data}}}
\left[
    r(x,y;\theta)
\right]
\in
\left[
    \exp\bigl(-\beta F_{\max}(\theta)\bigr),
    1
\right]
\]
by Lemma~\ref{thm:lemma6} (iii).
Define the entropic risk
\[
R_\lambda(x,y;\theta)
:=
\frac{1}{\lambda}
\log M_\lambda(x,y;\theta),
\qquad \lambda>0,
\]
and
\[
R_0(x,y;\theta)
:=
\mathbb{E}_{X\sim p_{\mathrm{dis}}(\cdot;x)}
\left[
    f(X,y;\theta)
\right].
\]
The map $\lambda\mapsto R_\lambda(x,y;\theta)$
is continuous at $0$.
As in Theorem~\ref{thm:snis}, the Langevin sampler is treated as an exact
sampler of $p_{\mathrm{adv}}(\cdot;x,y,\theta)$
(non-asymptotic guarantees for projected Langevin dynamics);
all statements are at a fixed $\theta$.

\begin{lemma}[Tilting and volume identities]
\label{thm:lemma6}
Fix $(x,y)\in\mathcal{X}\times\mathcal{Y}$, $\theta$,
and $\beta\in[0,c_2]$, and let $\lambda=c_2-\beta$.
Then:
\begin{enumerate}
    \item[(i)]
    For every $t\geq 0$,
    \[
    1
    \leq M_t(x,y;\theta)
    \leq \exp\bigl(t F_{\max}(\theta)\bigr),
    \]
    and $p_t(\cdot;x,y,\theta)$ is a valid probability density.
    In particular,
    \[
    p_0(\cdot;x,y,\theta)
    =
    p_{\mathrm{dis}}(\cdot;x),
    \qquad
    p_{c_2}(\cdot;x,y,\theta)
    =
    p_{\mathrm{adv}}(\cdot;x,y,\theta),
    \]
    with
    \[
    Z_{\mathrm{adv}}Z_{\mathrm{vic}}
    =
    M_{c_2}(x,y;\theta),
    \]
    which is the identity used in Section~\ref{sec:batch}.

    \item[(ii)]
    Pointwise on $\mathcal{X}$,
    \[
    p_{\mathrm{adv}}(u;x,y,\theta)
    \exp\bigl(-\beta f(u,y;\theta)\bigr)
    =
    r(x,y;\theta)\,p_\lambda(u;x,y,\theta).
    \]

    \item[(iii)]
    \[
    \exp\bigl(-\beta F_{\max}(\theta)\bigr)
    \leq r(x,y;\theta)
    \leq 1,
    \]
    and
    \[
    r(x,y;\theta)^{-1}
    =
    \mathbb{E}_{U\sim p_\lambda(\cdot;x,y,\theta)}
    \left[
        \exp\bigl(\beta f(U,y;\theta)\bigr)
    \right].
    \]
\end{enumerate}
\end{lemma}

\begin{proof}
\textup{(i)}
The function $\exp\bigl(t f(\cdot,y;\theta)\bigr)$
is continuous on the compact set $\mathcal{X}$,
with values in
$[1,\exp(t F_{\max}(\theta))]$.
Hence
\[
1
\leq M_t(x,y;\theta)
\leq \exp\bigl(t F_{\max}(\theta)\bigr),
\]
and $p_t(\cdot;x,y,\theta)$ integrates to one.
The case $t=0$ gives
\[
p_0(\cdot;x,y,\theta)
=
p_{\mathrm{dis}}(\cdot;x).
\]
Since
\[
p_{\mathrm{adv}}(u;x,y,\theta)
=
\frac{
    p_{\mathrm{dis}}(u;x)
    \exp\bigl(c_2 f(u,y;\theta)\bigr)
}{
    Z_{\mathrm{adv}}Z_{\mathrm{vic}}
},
\]
integrating over $\mathcal{X}$ gives
\[
Z_{\mathrm{adv}}Z_{\mathrm{vic}}
=
M_{c_2}(x,y;\theta),
\]
and therefore
\[
p_{\mathrm{adv}}(\cdot;x,y,\theta)
=
p_{c_2}(\cdot;x,y,\theta).
\]

\textup{(ii)}
Using \textup{(i)} and $\lambda=c_2-\beta$,
\begin{align*}
& p_{\mathrm{adv}}(u;x,y,\theta)
  \exp\bigl(-\beta f(u,y;\theta)\bigr)
\\
&\qquad =
\frac{
    p_{\mathrm{dis}}(u;x)
    \exp\bigl(\lambda f(u,y;\theta)\bigr)
}{
    M_{c_2}(x,y;\theta)
}
\\
&\qquad =
\frac{
    M_\lambda(x,y;\theta)
}{
    M_{c_2}(x,y;\theta)
}
\cdot
\frac{
    p_{\mathrm{dis}}(u;x)
    \exp\bigl(\lambda f(u,y;\theta)\bigr)
}{
    M_\lambda(x,y;\theta)
}
\\
&\qquad =
r(x,y;\theta)\,p_\lambda(u;x,y,\theta).
\end{align*}

\textup{(iii)}
Since $f(u,y;\theta)\geq 0$ and $\lambda\leq c_2$,
\[
\exp\bigl(\lambda f(u,y;\theta)\bigr)
\leq
\exp\bigl(c_2 f(u,y;\theta)\bigr).
\]
Thus
\[
M_\lambda(x,y;\theta)
\leq
M_{c_2}(x,y;\theta),
\qquad
r(x,y;\theta)\leq 1.
\]
Conversely,
\begin{align*}
M_{c_2}(x,y;\theta)
&=
\mathbb{E}_{U\sim p_{\mathrm{dis}}(\cdot;x)}
\left[
    \exp\bigl(\lambda f(U,y;\theta)\bigr)
    \exp\bigl(\beta f(U,y;\theta)\bigr)
\right]
\\
&\leq
\exp\bigl(\beta F_{\max}(\theta)\bigr)
M_\lambda(x,y;\theta),
\end{align*}
so
\[
r(x,y;\theta)
\geq
\exp\bigl(-\beta F_{\max}(\theta)\bigr).
\]
Finally,
\[
\mathbb{E}_{U\sim p_\lambda(\cdot;x,y,\theta)}
\left[
    \exp\bigl(\beta f(U,y;\theta)\bigr)
\right]
=
\frac{
    M_{c_2}(x,y;\theta)
}{
    M_\lambda(x,y;\theta)
}
=
r(x,y;\theta)^{-1}.
\]
\end{proof}

Part~(iii) makes the mechanism transparent:
the local adversarial volume
$Z_{\mathrm{adv}}Z_{\mathrm{vic}}=M_{c_2}(x,y;\theta)$
omitted from the empirical weights resurfaces in the large-batch
limit as the per-image factor
\[
r(x,y;\theta)
=
\frac{
    M_\lambda(x,y;\theta)
}{
    M_{c_2}(x,y;\theta)
},
\]
which down-weights a clean pair exactly according to the
exponential moment of its loss over its perturbation neighborhood.
This is the precise form of the implicit regularization
described informally in Section~\ref{sec:batch}.

\begin{theorem}[Batch-level consistency of the empirical gradient estimator]
\label{thm:thm7}
Let
\[
(x^{(i)},y^{(i)})
\overset{\mathrm{i.i.d.}}{\sim}
p_{\mathrm{data}},
\qquad i=1,\ldots,B,
\]
and, conditionally on these clean pairs, let
\[
X^{(i)}
\sim
p_{\mathrm{adv}}
\bigl(\cdot;x^{(i)},y^{(i)},\theta\bigr)
\]
independently across $i$.
Equivalently, the tuples
$(x^{(i)},y^{(i)},X^{(i)})$ are i.i.d.\ from the joint density
\[
\pi(x,y,u)
:=
p_{\mathrm{data}}(x,y)\,
p_{\mathrm{adv}}(u;x,y,\theta).
\]
Let
\[
\hat g_B^{\mathrm{emp}}(\theta)
=
\sum_{i=1}^{B}
\hat w_{\mathrm{emp}}^{(i)}
\nabla_\theta f(X^{(i)},y^{(i)};\theta)
\]
be the update of Lines~16-19 of Algorithm~\ref{alg:full_training}, with
self-normalized weights
\[
\hat w_{\mathrm{emp}}^{(i)}
=
\frac{
    \tilde w_{\mathrm{emp}}^{(i)}
}{
    \sum_{j=1}^{B}\tilde w_{\mathrm{emp}}^{(j)}
},
\qquad
\tilde w_{\mathrm{emp}}^{(i)}
:=
\exp\bigl(-\beta f(X^{(i)},y^{(i)};\theta)\bigr),
\]
treated as constants during backpropagation
(stop-gradient, Line~17).
If
\[
\mathbb{E}_{(x,y)\sim p_{\mathrm{data}}}
\left[
    \mathbb{E}_{X\sim p_{\mathrm{adv}}(\cdot;x,y,\theta)}
    \left[
        \left\|\nabla_\theta f(X,y;\theta)\right\|
    \right]
\right]
<\infty,
\]
then, as $B\to\infty$,
\[
\hat g_B^{\mathrm{emp}}(\theta)
\xrightarrow{\mathrm{a.s.}}
\mathbb{E}_{(x,y)\sim\tilde p(\cdot,\cdot;\theta)}
\left[
    \mathbb{E}_{X\sim p_\lambda(\cdot;x,y,\theta)}
    \left[
        \nabla_\theta f(X,y;\theta)
    \right]
\right].
\]
\end{theorem}

\begin{proof}
Write $\hat g_B^{\mathrm{emp}}(\theta)$ as a ratio of sample means:
\[
\hat g_B^{\mathrm{emp}}(\theta)
=
\frac{
    \frac{1}{B}
    \sum_{i=1}^{B}
    \tilde w_{\mathrm{emp}}^{(i)}
    \nabla_\theta f(X^{(i)},y^{(i)};\theta)
}{
    \frac{1}{B}
    \sum_{j=1}^{B}
    \tilde w_{\mathrm{emp}}^{(j)}
}.
\]
Only independence across tuples is used.
It holds because the clean pairs are i.i.d.\ and, conditionally
on them, the $B$ Langevin chains evolve independently:
the batch energy in Line~11 is separable across $i$,
so its gradient decouples blockwise, and the noise in Line~9
is independent across $i$.
The within-tuple dependence is exactly what the factorization
\[
\pi(x,y,u)
=
p_{\mathrm{data}}(x,y)\,
p_{\mathrm{adv}}(u;x,y,\theta)
\]
encodes.

\medskip
\noindent\textit{Denominator.}
Since
\[
\tilde w_{\mathrm{emp}}^{(j)}
\in
\left[
    \exp\bigl(-\beta F_{\max}(\theta)\bigr),
    1
\right],
\]
the weights are integrable.
By the Strong Law of Large Numbers (SLLN), the tower property,
and Lemma~\ref{thm:lemma6} (ii),
\begin{align*}
\frac{1}{B}
\sum_{j=1}^{B}
\tilde w_{\mathrm{emp}}^{(j)}
&\xrightarrow{\mathrm{a.s.}}
\mathbb{E}_{(x,y,X)\sim\pi}
\left[
    \exp\bigl(-\beta f(X,y;\theta)\bigr)
\right]
\\
&=
\mathbb{E}_{(x,y)\sim p_{\mathrm{data}}}
\left[
    \mathbb{E}_{X\sim p_{\mathrm{adv}}(\cdot;x,y,\theta)}
    \left[
        \exp\bigl(-\beta f(X,y;\theta)\bigr)
    \right]
\right]
\\
&=
\mathbb{E}_{(x,y)\sim p_{\mathrm{data}}}
\left[
    r(x,y;\theta)
\right]
\\
&\geq
\exp\bigl(-\beta F_{\max}(\theta)\bigr)
>0,
\end{align*}
where the lower bound follows from Lemma~\ref{thm:lemma6} (iii).

\medskip
\noindent\textit{Numerator.}
Since
\[
\left\|
    \tilde w_{\mathrm{emp}}^{(i)}
    \nabla_\theta f(X^{(i)},y^{(i)};\theta)
\right\|
\leq
\left\|
    \nabla_\theta f(X^{(i)},y^{(i)};\theta)
\right\|,
\]
the integrability assumption allows the SLLN to be applied
componentwise:
\begin{align*}
&\frac{1}{B}
\sum_{i=1}^{B}
\tilde w_{\mathrm{emp}}^{(i)}
\nabla_\theta f(X^{(i)},y^{(i)};\theta)
\\
&\qquad\xrightarrow{\mathrm{a.s.}}
\mathbb{E}_{(x,y,X)\sim\pi}
\left[
    \exp\bigl(-\beta f(X,y;\theta)\bigr)
    \nabla_\theta f(X,y;\theta)
\right].
\end{align*}
By the tower property and Lemma~\ref{thm:lemma6} (ii),
\begin{align*}
&\mathbb{E}_{X\sim p_{\mathrm{adv}}(\cdot;x,y,\theta)}
\left[
    \exp\bigl(-\beta f(X,y;\theta)\bigr)
    \nabla_\theta f(X,y;\theta)
\right]
\\
&\qquad=
r(x,y;\theta)\,
\mathbb{E}_{X\sim p_\lambda(\cdot;x,y,\theta)}
\left[
    \nabla_\theta f(X,y;\theta)
\right].
\end{align*}
Thus the numerator converges almost surely to
\[
\mathbb{E}_{(x,y)\sim p_{\mathrm{data}}}
\left[
    r(x,y;\theta)\,
    \mathbb{E}_{X\sim p_\lambda(\cdot;x,y,\theta)}
    \left[
        \nabla_\theta f(X,y;\theta)
    \right]
\right].
\]

\medskip
\noindent\textit{Ratio.}
On the probability-one event where both limits hold,
the denominator tends to a strictly positive constant
by Lemma~\ref{thm:lemma6} (iii).
The continuous mapping theorem therefore gives
\begin{align*}
\hat g_B^{\mathrm{emp}}(\theta)
&\xrightarrow{\mathrm{a.s.}}
\frac{
    \mathbb{E}_{(x,y)\sim p_{\mathrm{data}}}
    \left[
        r(x,y;\theta)\,
        \mathbb{E}_{X\sim p_\lambda(\cdot;x,y,\theta)}
        \left[
            \nabla_\theta f(X,y;\theta)
        \right]
    \right]
}{
    \mathbb{E}_{(x,y)\sim p_{\mathrm{data}}}
    \left[
        r(x,y;\theta)
    \right]
}
\\
&=
\int_{\mathcal{X}\times\mathcal{Y}}
\frac{
    p_{\mathrm{data}}(x,y)\,r(x,y;\theta)
}{
    \mathbb{E}_{(x',y')\sim p_{\mathrm{data}}}
    \left[
        r(x',y';\theta)
    \right]
}
\\
&\qquad\qquad{}\times
\mathbb{E}_{X\sim p_\lambda(\cdot;x,y,\theta)}
\left[
    \nabla_\theta f(X,y;\theta)
\right]
\,\mathrm{d}(x,y)
\\
&=
\mathbb{E}_{(x,y)\sim\tilde p(\cdot,\cdot;\theta)}
\left[
    \mathbb{E}_{X\sim p_\lambda(\cdot;x,y,\theta)}
    \left[
        \nabla_\theta f(X,y;\theta)
    \right]
\right],
\end{align*}
where the last equality holds because the normalizing constant of
\[
\tilde p(x,y;\theta)
\propto
p_{\mathrm{data}}(x,y)\,r(x,y;\theta)
\]
is exactly
\[
\mathbb{E}_{(x,y)\sim p_{\mathrm{data}}}
\left[
    r(x,y;\theta)
\right].
\]
\end{proof}

\noindent\textbf{Relation to Theorem~\ref{thm:snis}.}
\begin{enumerate}
    \item[(i)]
    The integrability assumption is implied by the data-averaged
    assumption of Theorem~\ref{thm:snis}.
    By Lemma~\ref{thm:lemma6} (i),
    \[
    \frac{
        p_{\mathrm{adv}}(u;x,y,\theta)
    }{
        p_{\mathrm{dis}}(u;x)
    }
    =
    \frac{
        \exp\bigl(c_2 f(u,y;\theta)\bigr)
    }{
        M_{c_2}(x,y;\theta)
    }
    \leq
    \exp\bigl(c_2 F_{\max}(\theta)\bigr),
    \]
    since $M_{c_2}(x,y;\theta)\geq 1$.
    Hence
    \begin{align*}
    &\mathbb{E}_{(x,y,X)\sim\pi}
    \left[
        \left\|\nabla_\theta f(X,y;\theta)\right\|
    \right]
    \\
    &\qquad\leq
    \exp\bigl(c_2 F_{\max}(\theta)\bigr)\,
    \mathbb{E}_{(x,y)\sim p_{\mathrm{data}}}
    \left[
        \mathbb{E}_{X\sim p_{\mathrm{dis}}(\cdot;x)}
        \left[
            \left\|\nabla_\theta f(X,y;\theta)\right\|
        \right]
    \right].
    \end{align*}

    \item[(ii)]
    \textbf{Theorem~\ref{thm:thm7} strictly generalizes Theorem~\ref{thm:snis}:}
    with a single clean pair
    $(x_{\mathrm{ori}},y_{\mathrm{ori}})$ and $\beta=c_2$,
    the factor $r(x_{\mathrm{ori}},y_{\mathrm{ori}};\theta)$
    is a common constant that cancels in the self-normalized ratio.
    Moreover, $\lambda=0$ and
    \[
    p_0(\cdot;x_{\mathrm{ori}},y_{\mathrm{ori}},\theta)
    =
    p_{\mathrm{dis}}(\cdot;x_{\mathrm{ori}}).
    \]
    The limit therefore reduces to
    \[
    \mathbb{E}_{X\sim p_{\mathrm{dis}}(\cdot;x_{\mathrm{ori}})}
    \left[
        \nabla_\theta f(X,y_{\mathrm{ori}};\theta)
    \right],
    \]
    which is the conclusion of Theorem~\ref{thm:snis}.

    \item[(iii)]
    At the other endpoint $\beta=0$, the weights are uniform:
    \[
    \hat w_{\mathrm{emp}}^{(i)}=\frac{1}{B}.
    \]
    In this case,
    \[
    r(x,y;\theta)=1,
    \qquad
    \tilde p(x,y;\theta)=p_{\mathrm{data}}(x,y),
    \]
    and
    \[
    p_\lambda(\cdot;x,y,\theta)
    =
    p_{\mathrm{adv}}(\cdot;x,y,\theta).
    \]
    This is unweighted adversarial training on Langevin samples,
    i.e., exactly the PAT(WOS) variant of Section~\ref{sec:ablation}.

    \item[(iv)]
    Had the intractable weights
    \[
    w_{\mathrm{true}}^{(i)}
    \propto
    Z_{\mathrm{adv}}^{(i)}Z_{\mathrm{vic}}^{(i)}
    \exp\bigl(-c_2 f(X^{(i)},y^{(i)};\theta)\bigr)
    \]
    of Section~\ref{sec:batch} been used, the same argument, using the identity
    \[
    p_{\mathrm{adv}}(u;x,y,\theta)\,
    Z_{\mathrm{adv}}Z_{\mathrm{vic}}\,
    \exp\bigl(-c_2 f(u,y;\theta)\bigr)
    =
    p_{\mathrm{dis}}(u;x)
    \]
    from Lemma~\ref{thm:lemma6}, would give the limit
    \[
    \mathbb{E}_{(x,y)\sim p_{\mathrm{data}}}
    \left[
        \mathbb{E}_{X\sim p_{\mathrm{dis}}(\cdot;x)}
        \left[
            \nabla_\theta f(X,y;\theta)
        \right]
    \right],
    \]
    the exact joint analogue of Theorem~\ref{thm:snis}.
\end{enumerate}

\begin{proposition}[Entropic-risk gradient identity]
\label{thm:prop8}
Fix $(x,y)\in\mathcal{X}\times\mathcal{Y}$.
Assume, as in Proposition~\ref{thm:prop4}, that $\nabla_\theta$ may be
exchanged with integration over $\mathcal{X}$.
Then, for $\lambda>0$,
\[
\nabla_\theta R_\lambda(x,y;\theta)
=
\mathbb{E}_{X\sim p_\lambda(\cdot;x,y,\theta)}
\left[
    \nabla_\theta f(X,y;\theta)
\right].
\]
The identity extends to $\lambda=0$, with
\[
R_0(x,y;\theta)
=
\mathbb{E}_{X\sim p_{\mathrm{dis}}(\cdot;x)}
\left[
    f(X,y;\theta)
\right],
\qquad
p_0(\cdot;x,y,\theta)
=
p_{\mathrm{dis}}(\cdot;x).
\]
\end{proposition}

\begin{proof}
Since $p_{\mathrm{dis}}(\cdot;x)$ does not depend on $\theta$,
\begin{align*}
\nabla_\theta \log M_\lambda(x,y;\theta)
&=
\frac{\lambda}{M_\lambda(x,y;\theta)}
\int_{\mathcal{X}}
p_{\mathrm{dis}}(u;x)\,
\nabla_\theta f(u,y;\theta)\,
\exp\bigl(\lambda f(u,y;\theta)\bigr)
\,\mathrm{d}u
\\
&=
\lambda\,
\mathbb{E}_{X\sim p_\lambda(\cdot;x,y,\theta)}
\left[
    \nabla_\theta f(X,y;\theta)
\right].
\end{align*}
Dividing by $\lambda$ gives the claim for $\lambda>0$.
\end{proof}

\begin{corollary}[What PAT optimizes in practice]
\label{thm:corollary9}
Combining Theorem~\ref{thm:thm7} and Proposition~\ref{thm:prop8}, as $B\to\infty$,
\[
\hat g_B^{\mathrm{emp}}(\theta)
\xrightarrow{\mathrm{a.s.}}
\left.
\nabla_{\theta'}
\mathbb{E}_{(x,y)\sim\tilde p(\cdot,\cdot;\theta)}
\left[
    R_\lambda(x,y;\theta')
\right]
\right|_{\theta'=\theta},
\]
where the gradient passes inside the outer expectation because
$\tilde p(\cdot,\cdot;\theta)$ does not depend on $\theta'$.
\end{corollary}

\begin{proposition}[The smoothed objective preserves the lower bound
on probabilistic robustness]
\label{thm:prop10}
Fix a clean pair $(x_{\mathrm{ori}},y_{\mathrm{ori}})$,
and let $s$, $E_{\mathrm{adv}}$, and $\gamma^*>0$
be as in Theorem~\ref{thm:kl_bound}.
Then:
\begin{enumerate}
    \item[(i)]
    \textit{Monotone family.}
    The map
    $\lambda\mapsto
    R_\lambda(x_{\mathrm{ori}},y_{\mathrm{ori}};\theta)$
    is continuous on $[0,\infty)$, and, for $\lambda>0$,
    \begin{align*}
    \frac{\mathrm{d}}{\mathrm{d}\lambda}
    R_\lambda(x_{\mathrm{ori}},y_{\mathrm{ori}};\theta)
    =
    \frac{1}{\lambda^2}
    \operatorname{KL}\!\left(
        p_\lambda(\cdot;x_{\mathrm{ori}},y_{\mathrm{ori}},\theta)
        \,\middle\|\,
        p_{\mathrm{dis}}(\cdot;x_{\mathrm{ori}})
    \right)
    \geq 0.
    \end{align*}
    Hence, for every $\lambda\in[0,c_2]$,
    \[
    R_0(x_{\mathrm{ori}},y_{\mathrm{ori}};\theta)
    \leq
    R_\lambda(x_{\mathrm{ori}},y_{\mathrm{ori}};\theta)
    \leq
    R_{c_2}(x_{\mathrm{ori}},y_{\mathrm{ori}};\theta),
    \]
    with the exact integral representation
    \begin{align*}
    R_\lambda(x_{\mathrm{ori}},y_{\mathrm{ori}};\theta)
    -
    R_0(x_{\mathrm{ori}},y_{\mathrm{ori}};\theta)
    =
    \int_0^\lambda
    \frac{1}{t^2}
    \operatorname{KL}\!\left(
        p_t(\cdot;x_{\mathrm{ori}},y_{\mathrm{ori}},\theta)
        \,\middle\|\,
        p_{\mathrm{dis}}(\cdot;x_{\mathrm{ori}})
    \right)
    \,\mathrm{d}t.
    \end{align*}

    \item[(ii)]
    \textit{A family of lower bounds.}
    For every $\beta\in[0,c_2]$, i.e.,
    every $\lambda=c_2-\beta\in[0,c_2]$,
    \[
    \operatorname{PR}\!\left(
        p_{\mathrm{dis}}(\cdot;x_{\mathrm{ori}}),s
    \right)
    \geq
    1-
    \frac{
        c_2 R_\lambda(x_{\mathrm{ori}},y_{\mathrm{ori}};\theta)
    }{
        \gamma^*
    }.
    \]
    At $\lambda=0$ (i.e., $\beta=c_2$),
    this coincides with the bound of Theorem~2, since
    \begin{align*}
    \log Z_{\mathrm{vic}}
    -
    H\!\left(p_{\mathrm{dis}}(\cdot;x_{\mathrm{ori}})\right)&-
    \operatorname{KL}\!\left(
        p_{\mathrm{dis}}(\cdot;x_{\mathrm{ori}})
        \,\middle\|\,
        p_{\mathrm{vic}}(\cdot;y_{\mathrm{ori}},\theta)
    \right)
    \\
    &=
    \mathbb{E}_{X\sim p_{\mathrm{dis}}(\cdot;x_{\mathrm{ori}})}
    \left[
        c_2 f(X,y_{\mathrm{ori}};\theta)
    \right]
    \\
    &=
    c_2 R_0(x_{\mathrm{ori}},y_{\mathrm{ori}};\theta).
    \end{align*}
    For $\lambda>0$, the bound is looser by exactly
    \[
    \frac{c_2}{\gamma^*}
    \left(
        R_\lambda(x_{\mathrm{ori}},y_{\mathrm{ori}};\theta)
        -
        R_0(x_{\mathrm{ori}},y_{\mathrm{ori}};\theta)
    \right),
    \]
    which is quantified by the integral representation
    in \textup{(i)}.
\end{enumerate}
\end{proposition}

\begin{proof}
\textup{(i)}
The boundedness of $f(\cdot,y_{\mathrm{ori}};\theta)$
justifies differentiation under the integral:
\[
\frac{\mathrm{d}}{\mathrm{d}\lambda}
\log M_\lambda(x_{\mathrm{ori}},y_{\mathrm{ori}};\theta)
=
\mathbb{E}_{
    X\sim p_\lambda(\cdot;x_{\mathrm{ori}},y_{\mathrm{ori}},\theta)
}
\left[
    f(X,y_{\mathrm{ori}};\theta)
\right].
\]
Moreover,
\begin{align*}
\operatorname{KL}\!\left(
    p_\lambda(\cdot;x_{\mathrm{ori}},y_{\mathrm{ori}},\theta)
    \,\middle\|\,
    p_{\mathrm{dis}}(\cdot;x_{\mathrm{ori}})
\right)
&=
\mathbb{E}_{
    X\sim p_\lambda(\cdot;x_{\mathrm{ori}},y_{\mathrm{ori}},\theta)
}
\left[
    \lambda f(X,y_{\mathrm{ori}};\theta)
    -
    \log M_\lambda(x_{\mathrm{ori}},y_{\mathrm{ori}};\theta)
\right]
\\
&=
\lambda
\frac{\mathrm{d}}{\mathrm{d}\lambda}
\log M_\lambda(x_{\mathrm{ori}},y_{\mathrm{ori}};\theta)
-
\log M_\lambda(x_{\mathrm{ori}},y_{\mathrm{ori}};\theta)
\geq 0.
\end{align*}
The quotient rule therefore gives
\begin{align*}
\frac{\mathrm{d}}{\mathrm{d}\lambda}
\left(
    \frac{
        \log M_\lambda(x_{\mathrm{ori}},y_{\mathrm{ori}};\theta)
    }{
        \lambda
    }
\right)
=
\frac{1}{\lambda^2}
\operatorname{KL}\!\left(
    p_\lambda(\cdot;x_{\mathrm{ori}},y_{\mathrm{ori}},\theta)
    \,\middle\|\,
    p_{\mathrm{dis}}(\cdot;x_{\mathrm{ori}})
\right).
\end{align*}
Continuity at $\lambda=0$ follows from the expansion
\[
\log M_\lambda(x_{\mathrm{ori}},y_{\mathrm{ori}};\theta)
=
\lambda\,
\mathbb{E}_{X\sim p_{\mathrm{dis}}(\cdot;x_{\mathrm{ori}})}
\left[
    f(X,y_{\mathrm{ori}};\theta)
\right]
+
O(\lambda^2).
\]
The integral representation follows from the fundamental theorem
of calculus, with the integrand integrable at $0$ since
$$
\frac{1}{t^2}
\operatorname{KL}\!\left(
    p_t(\cdot;x_{\mathrm{ori}},y_{\mathrm{ori}},\theta)
    \,\middle\|\,
    p_{\mathrm{dis}}(\cdot;x_{\mathrm{ori}})
\right)
\xrightarrow[t\downarrow 0]{}
\frac{1}{2}
\operatorname{Var}_{X\sim p_{\mathrm{dis}}(\cdot;x_{\mathrm{ori}})}
\left[
    f(X,y_{\mathrm{ori}};\theta)
\right].
$$

\medskip
\noindent\textup{(ii)}
Jensen's inequality gives
\[
M_\lambda(x_{\mathrm{ori}},y_{\mathrm{ori}};\theta)
\geq
\exp\!\left(
    \lambda\,
    \mathbb{E}_{X\sim p_{\mathrm{dis}}(\cdot;x_{\mathrm{ori}})}
    \left[
        f(X,y_{\mathrm{ori}};\theta)
    \right]
\right),
\]
and hence
\[
R_\lambda(x_{\mathrm{ori}},y_{\mathrm{ori}};\theta)
\geq
R_0(x_{\mathrm{ori}},y_{\mathrm{ori}};\theta).
\]
From the proof of Theorem~2,
\[
\mathbb{E}_{X\sim p_{\mathrm{dis}}(\cdot;x_{\mathrm{ori}})}
\left[
    c_2 f(X,y_{\mathrm{ori}};\theta)
\right]
\geq
\gamma^*\,
I\!\left[
    p_{\mathrm{dis}}(\cdot;x_{\mathrm{ori}}),s
\right].
\]
Therefore,
\[
I\!\left[
    p_{\mathrm{dis}}(\cdot;x_{\mathrm{ori}}),s
\right]
\leq
\frac{
    c_2 R_0(x_{\mathrm{ori}},y_{\mathrm{ori}};\theta)
}{
    \gamma^*
}
\leq
\frac{
    c_2 R_\lambda(x_{\mathrm{ori}},y_{\mathrm{ori}};\theta)
}{
    \gamma^*
},
\]
and
\begin{align*}
\operatorname{PR}\!\left(
    p_{\mathrm{dis}}(\cdot;x_{\mathrm{ori}}),s
\right)
&=
1-
I\!\left[
    p_{\mathrm{dis}}(\cdot;x_{\mathrm{ori}}),s
\right]
\geq
1-
\frac{
    c_2 R_\lambda(x_{\mathrm{ori}},y_{\mathrm{ori}};\theta)
}{
    \gamma^*
}.
\end{align*}
\end{proof}

\begin{corollary}[Aggregate control of average probabilistic robustness]
\label{thm:corollary11}
For each $y\in\mathcal{Y}$, define the per-label quantities
\begin{align*}
s_y(u)
&:=
\max_{y'\neq y}
\bigl(z_{y'}(u)-z_y(u)\bigr),
\\
E_{\mathrm{adv}}(y)
&:=
\left\{
    u\in\mathcal{X}:s_y(u)\geq 0
\right\},
\\
\gamma^*(y;\theta)
&:=
\inf_{u\in E_{\mathrm{adv}}(y)}
c_2 f(u,y;\theta),
\end{align*}
with the convention $\inf\varnothing=+\infty$.
Assume $\gamma^*(y;\theta)>0$ for every $y\in\mathcal{Y}$.
Since $\mathcal{Y}$ is finite,
\[
\gamma_{\min}^*(\theta)
:=
\min_{y\in\mathcal{Y}}
\gamma^*(y;\theta)
>0.
\]
Define the average probabilistic robustness
\[
\overline{\operatorname{PR}}(\theta)
:=
\mathbb{E}_{(x,y)\sim p_{\mathrm{data}}}
\left[
    \operatorname{PR}\!\left(
        p_{\mathrm{dis}}(\cdot;x),s_y
    \right)
\right].
\]
Then, for every $\beta\in[0,c_2]$,
\[
\overline{\operatorname{PR}}(\theta)
\geq
1-
\frac{
    c_2\exp\bigl(\beta F_{\max}(\theta)\bigr)
}{
    \gamma_{\min}^*(\theta)
}
\mathbb{E}_{(x,y)\sim\tilde p(\cdot,\cdot;\theta)}
\left[
    R_\lambda(x,y;\theta)
\right].
\]
\end{corollary}

\begin{proof}
Apply Proposition~10(ii) to each clean pair.
Using
\[
\gamma^*(y;\theta)\geq\gamma_{\min}^*(\theta)
\]
and $R_\lambda(x,y;\theta)\geq 0$
(since $M_\lambda(x,y;\theta)\geq 1$),
and taking expectations over $(x,y)\sim p_{\mathrm{data}},
$ gives
\[
\overline{\operatorname{PR}}(\theta)
\geq
1-
\frac{c_2}{\gamma_{\min}^*(\theta)}
\mathbb{E}_{(x,y)\sim p_{\mathrm{data}}}
\left[
    R_\lambda(x,y;\theta)
\right].
\]
By Lemma~\ref{thm:lemma6}(iii),
\[
\exp\bigl(-\beta F_{\max}(\theta)\bigr)
\leq
r(x,y;\theta)
\leq 1.
\]
Hence,
\begin{align*}
\mathbb{E}_{(x,y)\sim\tilde p(\cdot,\cdot;\theta)}
\left[
    R_\lambda(x,y;\theta)
\right]
&=
\frac{
    \mathbb{E}_{(x,y)\sim p_{\mathrm{data}}}
    \left[
        r(x,y;\theta)\,R_\lambda(x,y;\theta)
    \right]
}{
    \mathbb{E}_{(x,y)\sim p_{\mathrm{data}}}
    \left[
        r(x,y;\theta)
    \right]
}
\\
&\geq
\exp\bigl(-\beta F_{\max}(\theta)\bigr)\,
\mathbb{E}_{(x,y)\sim p_{\mathrm{data}}}
\left[
    R_\lambda(x,y;\theta)
\right].
\end{align*}
Equivalently,
\begin{align*}
&\mathbb{E}_{(x,y)\sim p_{\mathrm{data}}}
\left[
    R_\lambda(x,y;\theta)
\right]
\leq
\exp\bigl(\beta F_{\max}(\theta)\bigr)\,
\mathbb{E}_{(x,y)\sim\tilde p(\cdot,\cdot;\theta)}
\left[
    R_\lambda(x,y;\theta)
\right].
\end{align*}
Combining the two bounds completes the proof.
\end{proof}

\section{Further Discussion on the Global Vulnerability Capacity $Z_{vic}$}
\label{app:discussion}

While our theoretical results establish $KL(p_{dis}||p_{vic}) - \log Z_{vic}$ as a surrogate for probabilistic robustness, a deeper understanding of the term $-\log Z_{vic}$ reveals a fundamental trade-off in robust optimization. From the perspective of distribution expansion, the normalization constant $Z_{vic} = \int_{\mathcal{X}} \exp(c_2 f(u, y_{ori}; \theta)) du$ represents the total ``volume'' of the model's vulnerability across the entire input hypercube $\mathcal{X} = [0,1]^d$.

Geometrically, the probability density $p_{vic}$ is a normalized representation of the classification loss. Because the domain $\mathcal{X}$ is compact and the integral of $p_{vic}$ is constrained to $1$, any attempt to minimize the density (i.e., increase robustness) in the local neighborhood of $x_{ori}$ via $KL(p_{dis}||p_{vic})$ must be compensated by an increase in density elsewhere in the domain. This is analogous to the ``water level'' in a closed pool: pushing the water away from one area (the distance distribution) necessarily causes the overall water level ($Z_{vic}$) to rise if the container's volume is fixed.

This perspective implies that $-\log Z_{vic}$ is not merely a mathematical residue but a global regularizer that penalizes the ``compression'' of vulnerabilities. If the optimization focuses solely on the KL term, the model might achieve high local robustness by creating extremely sharp decision boundaries just outside the $p_{dis}$ support, leading to a surge in $Z_{vic}$ and potential global instability. Therefore, PAT's inherent structure forces a principled balance: it requires the model to not only push vulnerabilities away from the clean data but also to suppress the total growth of the global vulnerability volume. This justifies the use of the smoothing factor $\beta$ in Equation \eqref{eq:wemp} as a mechanism to manage this pressure and prevent catastrophic weight collapse during the mass redistribution process.

\section{Details of the Experiments}
\label{app:details}

In this section, we provide the comprehensive experimental settings, hyperparameters, and evaluation protocols required to reproduce the results presented in Section~\ref{sec:exps}.

\subsection{Basic Training Setup}
For all experiments conducted on the CIFAR-10 and CIFAR-100 datasets, we apply a standard data augmentation pipeline to the training set: a 4-pixel reflection padding (\texttt{Pad(4, reflect)}), followed by a random crop to $32 \times 32$ pixels, and a random horizontal flip with a probability of $p=0.5$.

We train all models for a total of $E_{total} = 100$ epochs using a batch size of $B = 256$. The network parameters are optimized using Stochastic Gradient Descent (SGD) with a Nesterov momentum of $0.9$ and a weight decay of $5 \times 10^{-4}$. The initial learning rate is set to $\alpha = 0.01$ and is scheduled using a MultiStepLR decay strategy, which multiplies the learning rate by a decay factor of $\gamma = 0.1$ at milestones of epoch $75$ and epoch $90$.

\subsection{Hyperparameters of Probabilistic Adversarial Training}
For our proposed Probabilistic Adversarial Training (PAT) detailed in Algorithm 2, the generation of the adversarial distribution via Langevin Dynamics requires careful calibration. The precise hyperparameter configurations used across our main experiments are summarized below:

\begin{itemize}
    \item \textbf{Langevin Dynamics Steps ($T$):} $100$ iterations.
    \item \textbf{Langevin Step Size ($\eta$):} $0.3$.
    \item \textbf{Noise Scale ($\sigma$):} $0.001$, which ensures sufficient stochastic exploration within the local neighborhood.
    \item \textbf{Gradient Clipping Threshold ($\rho$):} $1.0$, applied to stabilize the energy landscape exploration.
    \item \textbf{Energy Weights:} We set the distance distribution weight to $c_1 = 0.3$ and the victim distribution (loss) weight to $c_2 = 0.42$.
    \item \textbf{Smoothing Factor ($\beta$):} As extensively discussed in Section 5.4 and Appendix \ref{app:discussion}, we instantiate the smoothing factor as $\beta = 0.001$ to prevent global weight collapse during the Self-Normalized Importance Sampling (SNIS).
\end{itemize}

\subsection{Baselines and Evaluation Protocol}
\textbf{Standard Robustness Evaluation:} To assess standard worst-case robustness, we evaluate the models using the well-established Projected Gradient Descent (PGD) and Carlini-Wagner (CW) attacks. Both attacks are executed for 20 iterative steps (PGD-20 and CW-20). $\epsilon=8/255, \alpha=2/255$. 

\textbf{Probabilistic Robustness Evaluation:} For the Probabilistic Robustness (PR) reported in Table 1, the test-time distance distributions are configured with escalating perturbation scales $\epsilon \in \{0.1, 0.12, 0.15, 0.2\}$. We sample $N=100$ perturbations per image to empirically estimate the PR objective.

\textbf{Baseline Training Settings:} For standard adversarial training baselines such as PGD-AT and TRADES, we strictly follow their original implementations. The baselines generate training adversaries using a 10-step attack with a maximum perturbation bounded by $L_\infty \le 8/255$.

\subsection{Compute Resources}
\label{app:compute}
All experiments, including baseline training and PAT, were implemented in PyTorch. The models were trained and evaluated on a compute cluster equipped with NVIDIA RTX 5090 GPUs. A single complete training run for PAT on CIFAR-10 takes approximately 7 minutes per epoch on a single GPU.

\section{Limitations}
Probabilistic adversarial attacks generate adversarial examples from noise and therefore typically require more generation steps than conventional adversarial attacks. This leads to a higher computational cost compared with standard attack methods.

\end{document}